%% file: main.tex
\documentclass{article} 
\usepackage{iclr2027_conference,times}
\iclrfinalcopy 
\input{math_commands.tex}

\usepackage{hyperref}
\usepackage{url}

\usepackage{booktabs}       
\usepackage{amsfonts}       
\usepackage{nicefrac}       
\usepackage{microtype}      
\usepackage{xcolor} 

\usepackage{amsmath}
\usepackage{bm} 
\usepackage{amssymb}
\usepackage{amsthm}
\usepackage{graphicx}

\usepackage{multirow}
\usepackage{makecell}
\usepackage{enumitem}
\usepackage{pifont}
\usepackage{subcaption}
\usepackage{float}

\usepackage{xurl}
\usepackage{soul}
\usepackage[table]{xcolor}
\definecolor{lightorange}{RGB}{255,242,204}
\sethlcolor{lightorange}

\usepackage{algorithm}
\usepackage{algpseudocode}

\usepackage{wrapfig}

\newtheorem{theorem}{Theorem}
\newtheorem{definition}{Definition}
\newtheorem{assumption}{Assumption}

\title{Preventing Rank Collapse in Federated Low-Rank Adaptation with Client Heterogeneity}

\author{
Fei Wu, 
Jia Hu\thanks{Corresponding author.}, 
Geyong Min\footnotemark[1], 
Shiqiang Wang \\
Department of Computer Science, University of Exeter, UK \\
\texttt{\{fw407,j.hu,g.min,s.wang9\}@exeter.ac.uk} \\
}

\begin{document}

\maketitle

\begin{abstract}
Federated low-rank adaptation (FedLoRA) has facilitated communication-efficient and privacy-preserving fine-tuning of foundation models for downstream tasks. In practical federated learning scenarios, client heterogeneity in system resources and data distributions motivates the use of heterogeneous LoRA ranks across clients. However, we identify a previously overlooked phenomenon in heterogeneous FedLoRA with SVD-based allocation, termed \textit{rank collapse}, where the energy of the global update becomes concentrated in the minimum shared rank, resulting in suboptimal performance and high sensitivity to rank configurations. Through theoretical analysis, we reveal the underlying mechanism of rank collapse: a mismatch between rank-agnostic aggregation weights and rank-dependent client contributions, which systematically suppresses higher-rank updates at a geometric rate over rounds. Motivated by this insight, we propose raFLoRA, a rank-partitioned aggregation method that decomposes local updates into rank partitions and then aggregates each partition weighted by its effective client contributions. Extensive experiments across vision, language, and reasoning tasks show that raFLoRA prevents rank collapse, improves model performance, and enhances robustness across diverse heterogeneous configurations compared with strong FedLoRA baselines.
\end{abstract}

\section{Introduction}
\label{sec:intro}
\input{Sections/0-introduction}

\section{Related Work}
\label{sec:relwork}
\input{Sections/1-related_work}

\section{Preliminaries}
\label{sec:preli}
\input{Sections/2-preliminaries}

\section{Problem Analysis}
\label{sec:analysis}
\input{Sections/3-problem_analysis}

\section{raFLoRA Method}
\label{sec:method}
\input{Sections/4-methodology}

\section{Experiments}
\label{sec:exp}
\input{Sections/5-experiments}

\section{Conclusion}
We identify rank collapse in heterogeneous FedLoRA with SVD-based allocation and provide a theoretical analysis showing that it results from a rank-wise averaging mismatch. To address this issue, we propose raFLoRA, a rank-partitioned aggregation strategy that aligns aggregation weights with rank-wise effective contributors. Extensive experiments across diverse tasks demonstrate its effectiveness and robustness, enabling efficient and scalable adaptation of large models.

\bibliography{iclr2027_conference}
\bibliographystyle{iclr2027_conference}

\appendix
\newpage
\input{Sections/appendix}

\end{document}

%% file: math_commands.tex
\usepackage{amsmath,amsfonts,bm}

\def\eqref#1{equation~\ref{#1}}

\def\1{\bm{1}}

\DeclareMathAlphabet{\mathsfit}{\encodingdefault}{\sfdefault}{m}{sl}
\SetMathAlphabet{\mathsfit}{bold}{\encodingdefault}{\sfdefault}{bx}{n}



%% file: Sections/0-introduction.tex

Pre-trained foundation models (FMs) have become the cornerstone of generative AI tasks across natural language processing (NLP) and computer vision (CV) domains \cite{vaswani2017transformer,grattafiori2024llama3,dosovitskiy2021vit}.
Fine-tuning these models has emerged as a standard paradigm for efficiently adapting them to diverse downstream tasks. 
However, high-quality public data are increasingly scarce, and privacy concerns over training on private data continue to grow. 
Federated learning (FL)~\cite{mcmahan2017communication} enables privacy-preserving collaboration to effectively leverage distributed private data for training high-quality models. Therefore, FL has been combined with fine-tuning of FMs in recent works \cite{hu2022lora, zhang2024fedit,chen2025rolora, guo2025fedsalora,sun2024ffalora, singhal2025fedexlora}.

However, the prohibitive communication overhead due to the scale of FMs has motivated the development of the federated low-rank adaptation (FedLoRA) framework~\cite{hu2022lora, zhang2024fedit}. 
Existing studies primarily focus on improving performance~\cite{chen2025rolora, guo2025fedsalora, bian2026fedalt} or mitigating aggregation bias~\cite{sun2024ffalora, singhal2025fedexlora}.
Despite their effectiveness, they typically assume homogeneous LoRA ranks across clients.

In practice, clients in FL exhibit inherent heterogeneity in both system resources and data distributions.
Client resources vary widely in computational capability, memory size, and communication bandwidth, while client data are typically non-independent and non-identically distributed (non-IID) across domains. 
In this context, the LoRA rank inherently controls the trade-off between resource usage and adaptation capacity.
For example, in centralized fine-tuning of LLaMA-3.1-8B~\cite{grattafiori2024llama3} on MetaMathQA40K~\cite{yu2024metamathqa40k} and evaluation on GSM8K~\cite{cobbe2021gsm8k}, increasing the LoRA rank from 8 to 256 enlarges the update size from 13\,MB to 416\,MB, while improving accuracy from 70.3\% to 74.1\%. 
Consequently, heterogeneous LoRA ranks offer a principled way to accommodate client heterogeneity in FL~\cite{tran2026ubsmoe}.

Several works have explored FedLoRA with heterogeneous ranks through aggregation and rank allocation across clients~\cite{cho2024hetlora, byun2025replication, wang2024flora, bai2024flexlora,xia2026sefora}. Among them, FlexLoRA~\cite{bai2024flexlora} leverages SVD-based allocation to avoid aggregation bias in HetLoRA~\cite{cho2024hetlora} caused by aggregating the $B$ and $A$ matrices separately, while maintaining communication efficiency. Figure~\ref{fig:fedlora} provides an overview of its framework.

\begin{figure}[t]
\begin{center}
\centerline{\includegraphics[width=0.7\columnwidth]{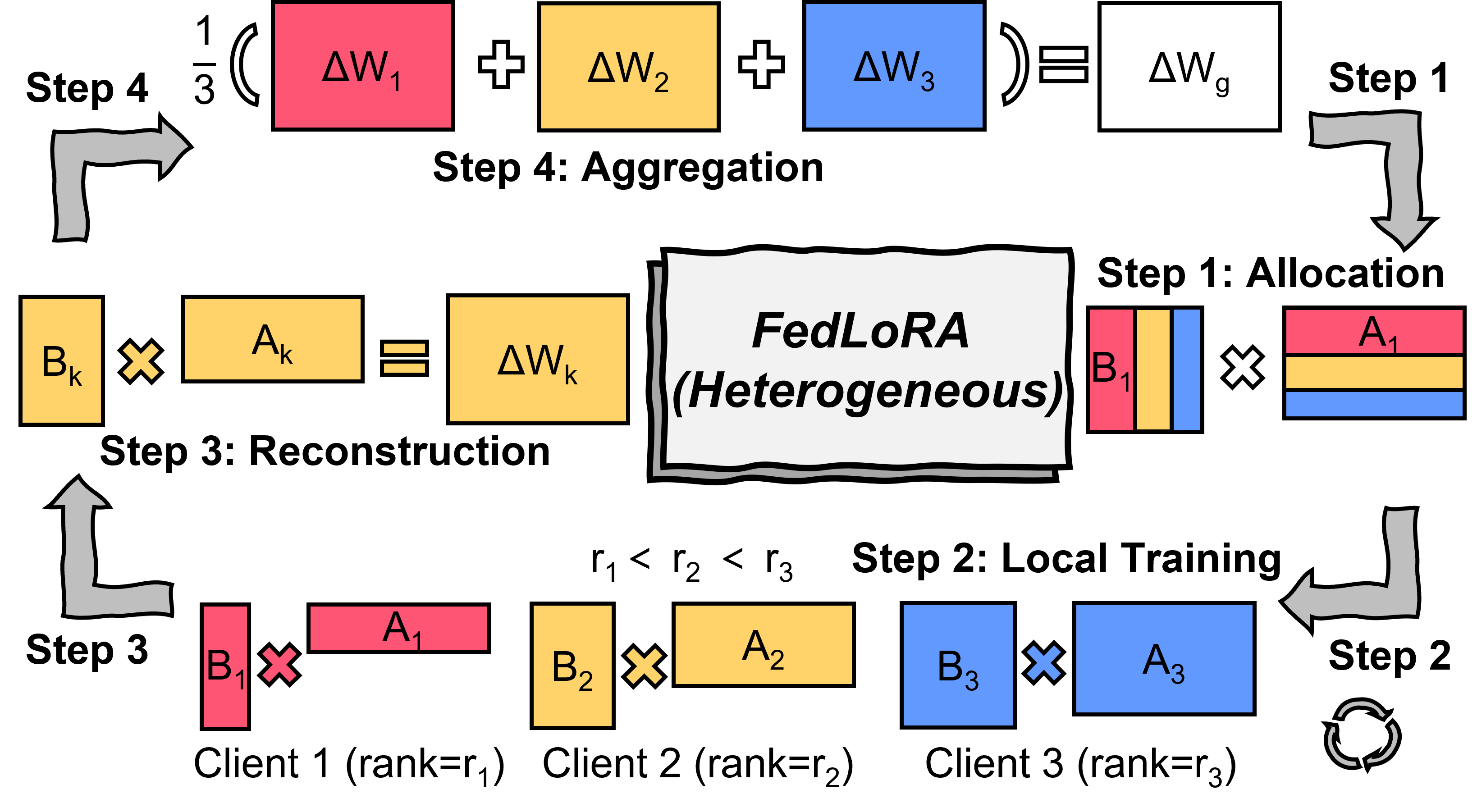}}
\caption{The overview of heterogeneous-rank FedLoRA with SVD-based allocation.}
\label{fig:fedlora}
\end{center}
\end{figure}

\begin{wrapfigure}{r}{0.58\textwidth}
\centering
\vspace{-0.2in}
\begin{subfigure}{0.49\linewidth}
\centering
\includegraphics[width=\linewidth]{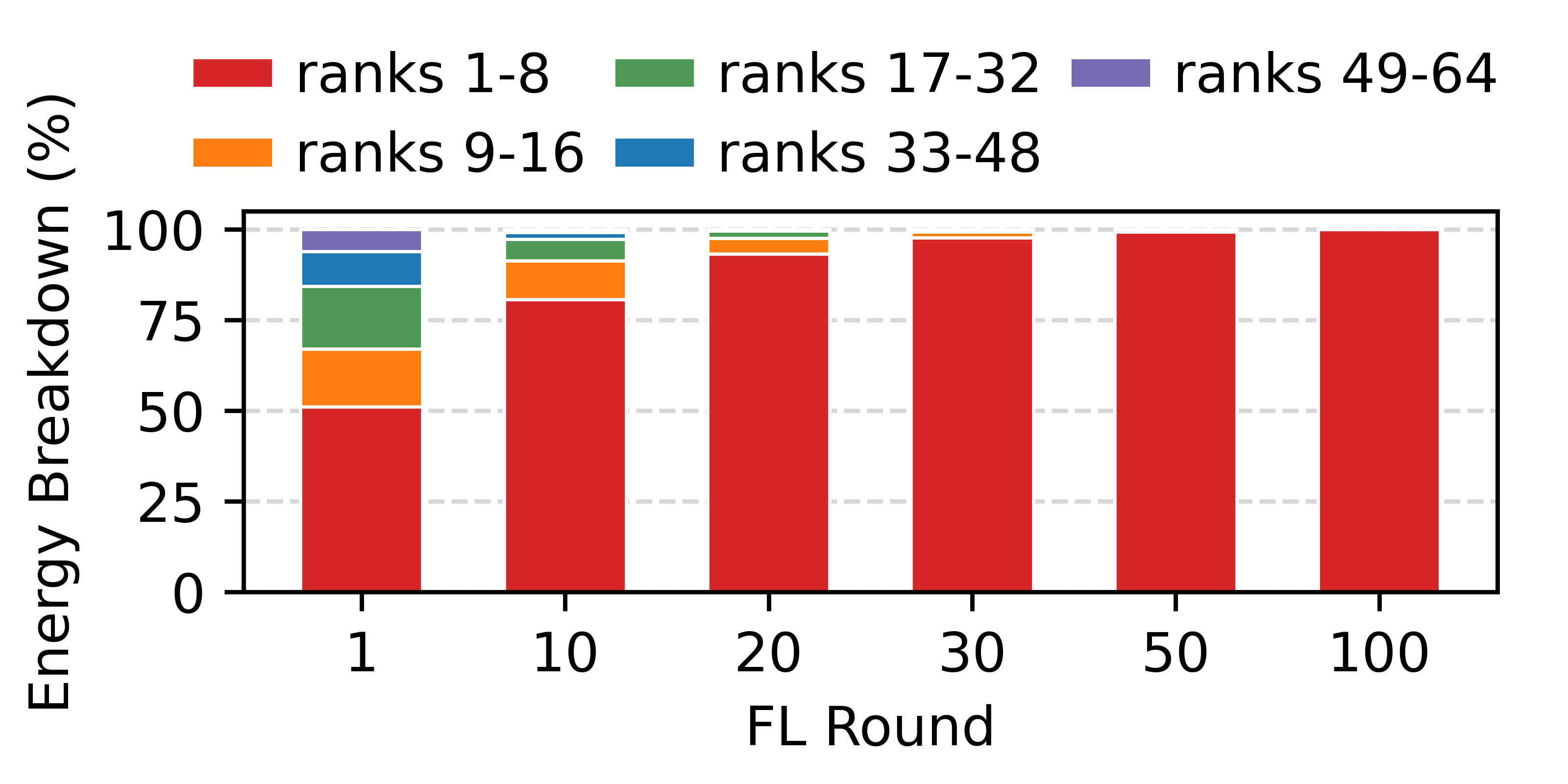}
\caption{FlexLoRA}
\label{fig:energy-flexlora}
\end{subfigure}
\hfill
\begin{subfigure}{0.49\linewidth}
\centering
\includegraphics[width=\linewidth]{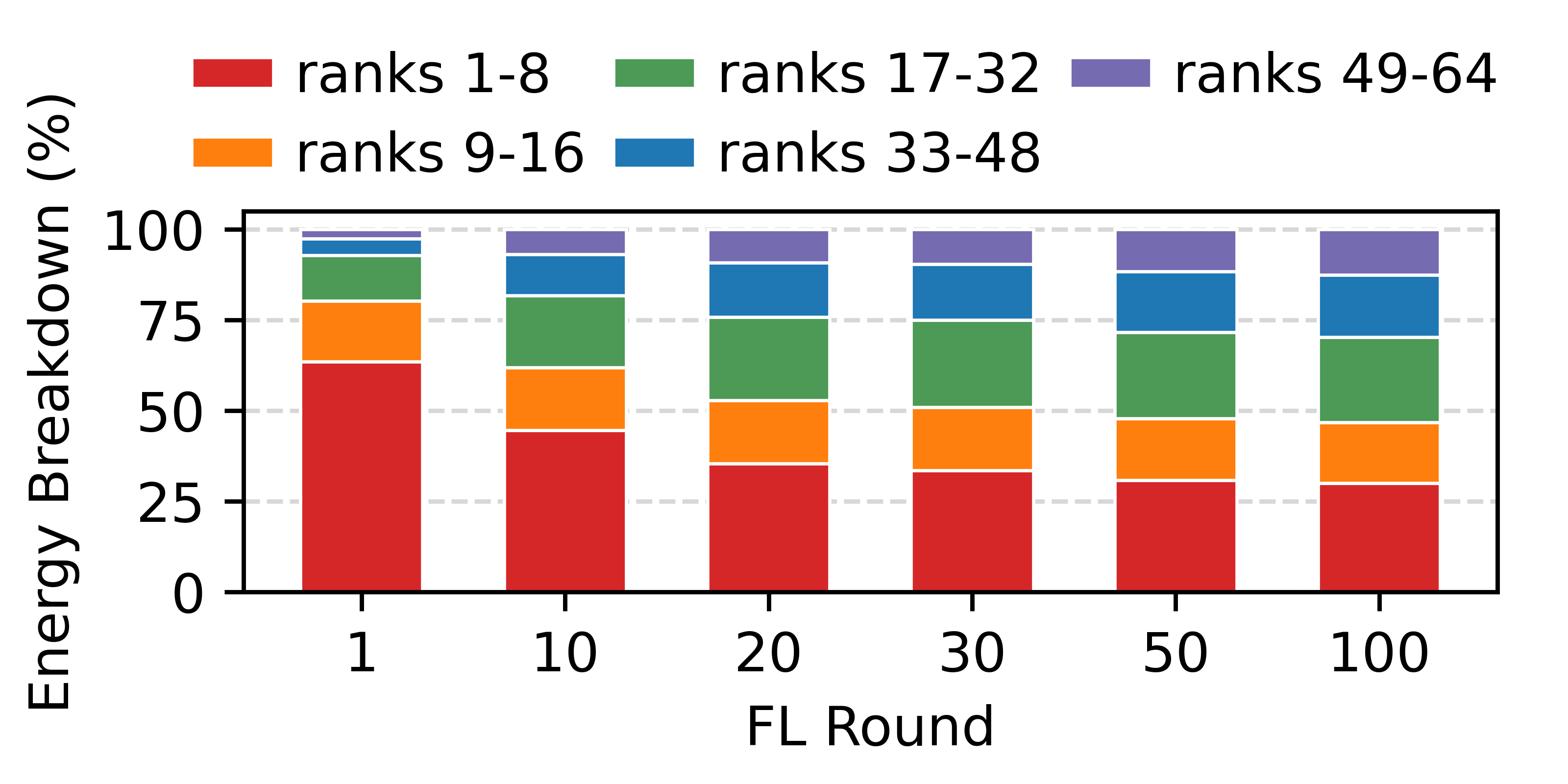}
\caption{raFLoRA (ours)}
\label{fig:energy-raflora}
\end{subfigure}
\begin{subfigure}{0.49\linewidth}
\centering
\includegraphics[width=\linewidth]{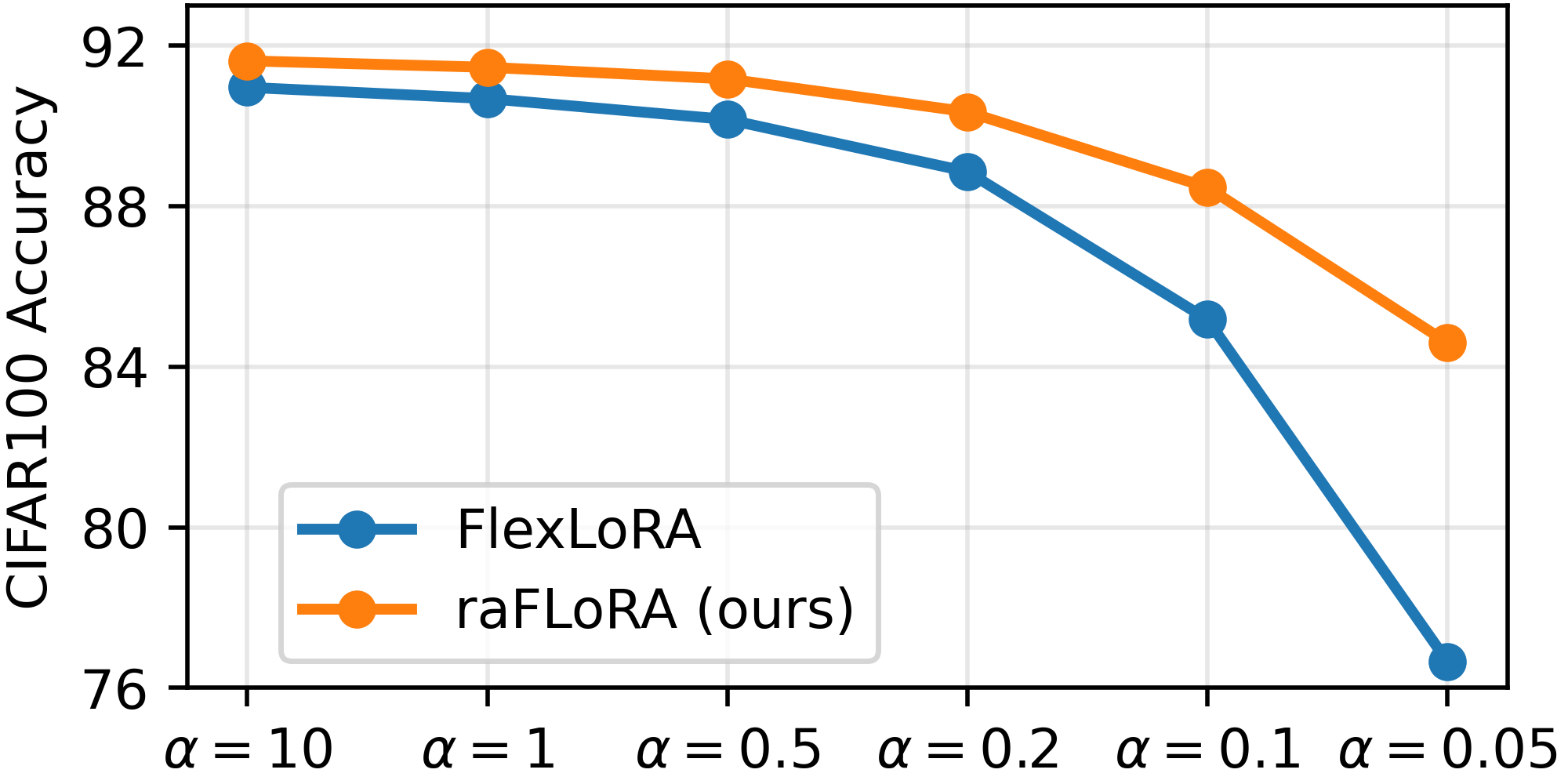}
\caption{non-IID settings}
\label{fig:acc-niid}
\end{subfigure}
\hfill
\begin{subfigure}{0.49\linewidth}
\centering
\includegraphics[width=\linewidth]{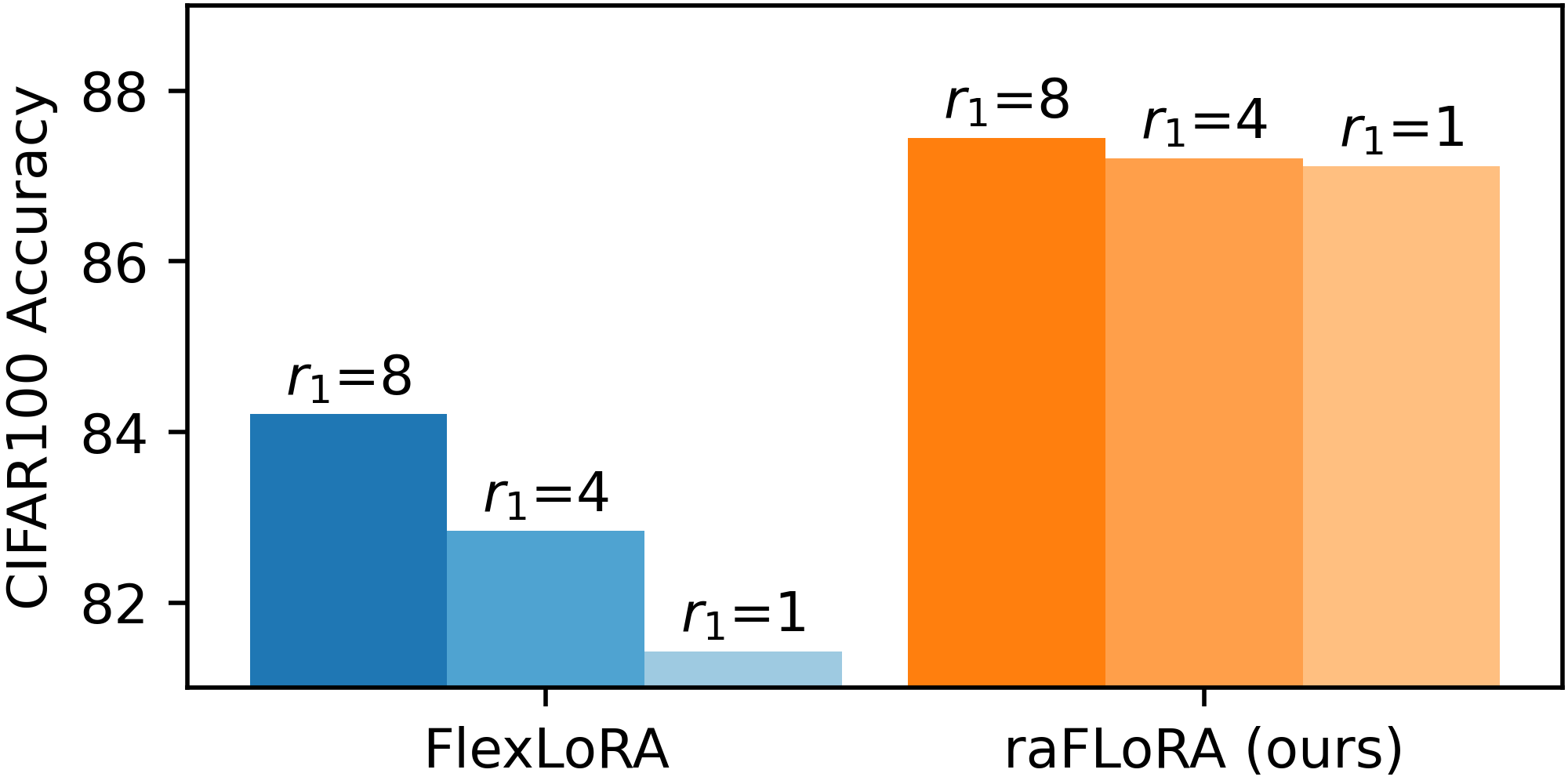}
\caption{Rank configurations}
\label{fig:acc-config}
\end{subfigure}
\caption{Energy breakdown of global update and accuracy under various settings.
In (a) and (b), the global update has an algebraic rank of 64. Client ranks are selected from $\{8,16,32,48,64\}$.
In (c), $\alpha$ controls the degree of data heterogeneity across clients.
In (d), only the shared rank $r_1$ varies across different settings, with details in Appendix~\ref{appendix:hyperparameters}. }
\label{fig:pre-results}
\vspace{-0.2in}
\end{wrapfigure}

Nevertheless, our observation identifies a previously overlooked phenomenon in the state-of-the-art FlexLoRA: 
although higher-rank clients contribute more parameters during training, almost all of the energy%
\footnote{In this paper, the term “energy” refers to the squared singular values of a matrix, not physical energy consumption at devices.}
of the global update is captured by the shared rank, as illustrated in Figure~\ref{fig:energy-flexlora}.
We term this behavior \textit{rank collapse}, where the collapse reflects the concentration of the energy spectrum, instead of a reduction in algebraic rank.
Consequently, the additional adaptation capacity introduced by higher-rank clients is not effectively translated into global model gains, leading to suboptimal performance across non-IID settings and high sensitivity to the shared rank, as shown in Figures~\ref{fig:acc-niid} and~\ref{fig:acc-config}.
Therefore, this study aims to address the following question:

\textit{How can we prevent rank collapse and improve the performance of heterogeneous-rank FedLoRA?}

Addressing this question poses two non-trivial challenges:
(i) Despite the empirical evidence, there remains a lack of principled understanding of why rank collapse arises under joint data and rank heterogeneity, where the aggregation and allocation processes across heterogeneous clients are tightly coupled.
(ii) Translating this understanding into a heterogeneous FedLoRA framework is challenging because preventing rank collapse may require the aggregation rule to account for how heterogeneous ranks shape client updates.
Together, these challenges call for a deeper understanding of rank collapse and a new rank-aware aggregation principle for heterogeneous-rank FedLoRA. 

To address these challenges, we conduct a theoretical analysis and reveal the underlying mechanism of rank collapse as a rank-wise averaging mismatch in heterogeneous-rank aggregation: the same aggregation weight is applied to all rank directions, while higher-rank directions (non-shared) are supported by only a subset of clients.
As a result, updates along higher-rank directions are systematically diluted under uniform averaging.
More importantly, this dilution accumulates across successive rounds, geometrically suppressing higher-rank contributions until they become negligible.

\begin{table}[t]
\caption{Comparison of federated low-rank adaptation methods. \(d\) and \(n\) denote the input and output dimensions, \(r\) denotes the LoRA rank, and \(M\) denotes the number of selected clients per round.}
\label{tab:intro-comparison}
\centering
\small
\begin{tabular}{lcccc}
\toprule
\textbf{Methods} 
& \makecell{\textbf{Heterogeneous} \\ \textbf{Rank Support}}
& \makecell{\textbf{Bias-Free} \\ \textbf{Aggregation}}
& \makecell{\textbf{Communication} \\ \textbf{Overhead}}
& \makecell{\textbf{Rank-aware} \\ \textbf{Aggregation}} \\
\midrule
FedIT~\cite{zhang2024fedit}       
& \ding{55} & \ding{55} & $\mathcal{O}((d+n)r)$ & N/A \\
HetLoRA~\cite{cho2024hetlora}   
& \ding{51} & \ding{55} & $\mathcal{O}((d+n)r)$ & \ding{55} \\
FLoRA~\cite{wang2024flora}       
& \ding{51} & \ding{51} & $\mathcal{O}(\min(M(d+n)r, dn))$ & \ding{55} \\
FlexLoRA~\cite{bai2024flexlora}
& \ding{51} & \ding{51} & $\mathcal{O}((d+n)r)$ & \ding{55} \\
\textbf{raFLoRA (ours)}
& \ding{51} & \ding{51} & $\mathcal{O}((d+n)r)$ & \ding{51} \\
\bottomrule
\end{tabular}
\end{table}

Motivated by this insight, we propose raFLoRA, a rank-partitioned aggregation method. Specifically, it decomposes local updates into non-overlapping rank partitions and aggregates each part using weights determined by its rank-wise effective contributors.
By correcting the mismatch, raFLoRA eliminates the per-round dilution of higher-rank updates and thereby prevents rank collapse.
Table~\ref{tab:intro-comparison} summarizes the key differences between raFLoRA and baselines.
Figure~\ref{fig:energy-raflora} shows raFLoRA reshapes the energy structure, and Figures~\ref{fig:acc-niid} and \ref{fig:acc-config} demonstrate its improved performance and enhanced robustness across diverse non-IID data distributions and various rank configurations.

We summarize the main contributions of this work below:
\begin{itemize}[leftmargin=*, itemsep=0pt, topsep=0pt]

\item We identify a previously overlooked phenomenon in SVD-based heterogeneous FedLoRA, termed \emph{rank collapse}, reveal its underlying mechanism as a mismatch between rank-agnostic aggregation weights and rank-dependent client contributions, and prove that the collapse proceeds geometrically.

\item We propose raFLoRA, a novel rank-partitioned aggregation method that resolves the mismatch by aligning aggregation weights with rank-wise effective client contributions.

\item We empirically demonstrate that raFLoRA effectively prevents rank collapse, consistently improves global model performance, and enhances robustness across diverse tasks and non-IID data distributions, while maintaining communication efficiency.

\end{itemize}

%% file: Sections/1-related_work.tex

\textbf{Homogeneous FedLoRA.}
To reduce communication in federated fine-tuning for FMs, FedIT~\cite{zhang2024fedit} integrates LoRA into local training.
Building on this, RoLoRA~\cite{chen2025rolora} improves convergence through alternating optimization, while FedSA-LoRA~\cite{guo2025fedsalora} decouples global and personalized LoRA into two components. 
To address data heterogeneity, SLoRA~\cite{babakniya2023slora} adopts a two-stage initialization strategy, and FRLoRA~\cite{yan2025frlora} increases effective update rank through residual-based updates. 
FedALT~\cite{bian2026fedalt} reduces cross-client interference via adaptive LoRA mixing.
To mitigate aggregation bias, FFA-LoRA~\cite{sun2024ffalora} updates only the LoRA $B$ matrix, and FedEx-LoRA~\cite{singhal2025fedexlora} applies local bias correction. 
However, these methods are limited to homogeneous FedLoRA settings.

\textbf{Heterogeneous FedLoRA.}
In practical FL scenarios, client heterogeneity motivates heterogeneous FedLoRA ranks across clients.
Zero-padding-based HetLoRA~\cite{cho2024hetlora} and replication-based padding~\cite{byun2025replication} enable aggregation across heterogeneous ranks, but introduce aggregation bias.
To address this issue, FLoRA~\cite{wang2024flora} adopts stacking-based aggregation to avoid factor-wise aggregation bias, at the cost of additional communication overhead and cold-start LoRA initialization.
SeFoRA~\cite{xia2026sefora} instead aggregates heterogeneous-rank updates via linear sketches.
FlexLoRA~\cite{bai2024flexlora} reconstructs full-size updates for aggregation and uses SVD to reassign rank-specific LoRA updates without additional communication overhead.
Fed-PLoRA~\cite{zhang2026fedplora}, ECLoRA~\cite{ning2025eclora}, and FedHL~\cite{peng2025fedhl} further focus on mitigating the bias introduced by SVD.
However, our analysis reveals that SVD-based heterogeneous FedLoRA suffers from rank collapse, which raFLoRA prevents through rank-partitioned aggregation to improve performance and robustness.

%% file: Sections/2-preliminaries.tex

In this section, we formalize heterogeneous FedLoRA with SVD-based rank allocation, using FlexLoRA~\cite{bai2024flexlora} as a representative framework that enables aggregation in the full parameter space while preserving communication efficiency.

Consider $K$ clients with heterogeneous LoRA ranks
$r_k \in \{r_1,r_2,\ldots,r_H\}$, where
$r_1<r_2<\cdots<r_H=r_{\max}$. At round $t$, the server uniformly samples $M$ clients $\mathcal{M}_t$. We define the \textit{rank coverage} of direction $i$ as $p_i=\mathbb{P}(r_k\ge i)=|\mathcal{K}_i|/K$, where $\mathcal{K}_i=\{k:r_k\ge i\}$, yielding
\begin{equation}
\label{eq:rank_coverage}
p_1=\cdots=p_{r_1}=1>p_{r_1+1}\ge\cdots\ge p_{r_{\max}}>0.
\end{equation}
Thus, higher-rank directions are supported by progressively fewer clients.

For a pre-trained layer $W_{\text{pre}}\in\mathbb{R}^{d\times n}$, the server maintains a rank-$r_{\max}$ global update $\Delta W_g^{(t)}\approx B_g^{(t)}A_g^{(t)}$, where $B_g^{(t)}\in\mathbb{R}^{d\times r_{\max}}$ and $A_g^{(t)}\in\mathbb{R}^{r_{\max}\times n}$. Each selected client $k$ receives the first $r_k$ components of the global LoRA factors, trains them on its local data, and returns $\Delta W_k^{(t)}=B_k^{(t)}A_k^{(t)}$.

The local updates are aggregated in the full $d\times n$ space using FedAvg~\cite{mcmahan2017communication}
\begin{equation}
\label{eq:aggregation}
\Delta W_g^{(t+1)}
=\frac{1}{M}\sum_{k\in\mathcal{M}_t}\Delta W_k^{(t)},
\end{equation}
where we assume equal client sample sizes for analytical simplicity. The aggregated update is then decomposed as $\Delta W_g^{(t+1)}\approx\sum_{i=1}^{r_{\max}}\sigma_i^{(t+1)}u_i^{(t+1)}v_i^{(t+1)\top}$ via truncated SVD~\cite{bai2024flexlora}, from which the global LoRA factors are reconstructed for the next round.

To characterize how heterogeneous ranks affect the global update, we define the expected energy of singular direction $i$ as $e_i^{(t)}=\mathbb{E}[(\sigma_i^{(t)})^2]$, the cumulative energy as $E_r^{(t)}=\sum_{i=1}^{r}e_i^{(t)}$, and the normalized energy ratio as $\rho_r^{(t)}=E_r^{(t)}/E_{r_{\max}}^{(t)}$. We refer to the first $r_1$ directions as the \textit{shared rank} and the remaining $r_{\max}-r_1$ directions as the \textit{higher ranks}. Accordingly, $\rho_{r_1}^{(t)}$ measures the energy fraction retained in the shared rank, while $1-\rho_{r_1}^{(t)}$ measures that in the higher ranks.

%% file: Sections/3-problem_analysis.tex

Based on the above formulation, we define rank collapse and analyze its emergence in SVD-based heterogeneous FedLoRA, providing theoretical insights into the underlying mechanism.

\subsection{Definition of Rank Collapse}
According to our observations in Figure ~\ref{fig:energy-flexlora}, although the global update is decomposed into rank-$r_{\max}$ components, its singular-value energy increasingly concentrates on the shared rank $r_1$ over rounds. This observation motivates the following definition:

\begin{definition}[Rank Collapse]
\label{def:erc}
In SVD-based heterogeneous FedLoRA, we define \textit{rank collapse} to occur when $1 - \rho_{r_1}^{(t)}$ progressively diminishes and becomes negligible over training rounds.
\end{definition}

Under rank collapse, although the global update retains algebraic rank $r_{\max}$ after SVD, its effective rank no longer reflects the available higher-rank capacity. Instead, the learning dynamics are governed by the shared rank $r_1$, rendering higher-rank components progressively ineffective.

\subsection{Theoretical Analysis: From Contributor Dilution to Rank Collapse}
\label{sec:theory}
We next identify the mechanism that drives rank collapse. Our analysis separates an exact one-round effect from its multi-round accumulation, and then extends the analysis to more general settings.

\paragraph{One-round contributor dilution.}
Let $r_0=0$ and define the $h$-th rank partition as
$\mathcal{I}_h=\{r_{h-1}+1,\ldots,r_h\}$.
For partition $h$, let
$\mathcal{C}_h^{(t)}=\{k\in\mathcal{M}_t:r_k\ge r_h\}$
and $C_h^{(t)}=|\mathcal{C}_h^{(t)}|$ denote its contributor set and contributor count.
When $C_h^{(t)}>0$, define its contributor average as
$\overline{\Delta W}_h^{(t)}
=(C_h^{(t)})^{-1}
\sum_{k\in\mathcal{C}_h^{(t)}}
B_k^{(t)}[:,\mathcal{I}_h]A_k^{(t)}[\mathcal{I}_h,:]$.
Rearranging Eq.~\ref{eq:aggregation} by rank partition gives
(see Appendix~\ref{appendix:one_round_dilution} for the full derivation)
\begin{equation}
\label{eq:contributor_dilution}
\Delta W_g^{(t+1)}
=
\sum_{h=1}^{H}
\frac{C_h^{(t)}}{M}
\overline{\Delta W}_h^{(t)},
\end{equation}
where the global update is decomposed into contributions from individual rank partitions

\begin{wrapfigure}{r}{0.52\columnwidth}
\centering
\includegraphics[width=\linewidth]{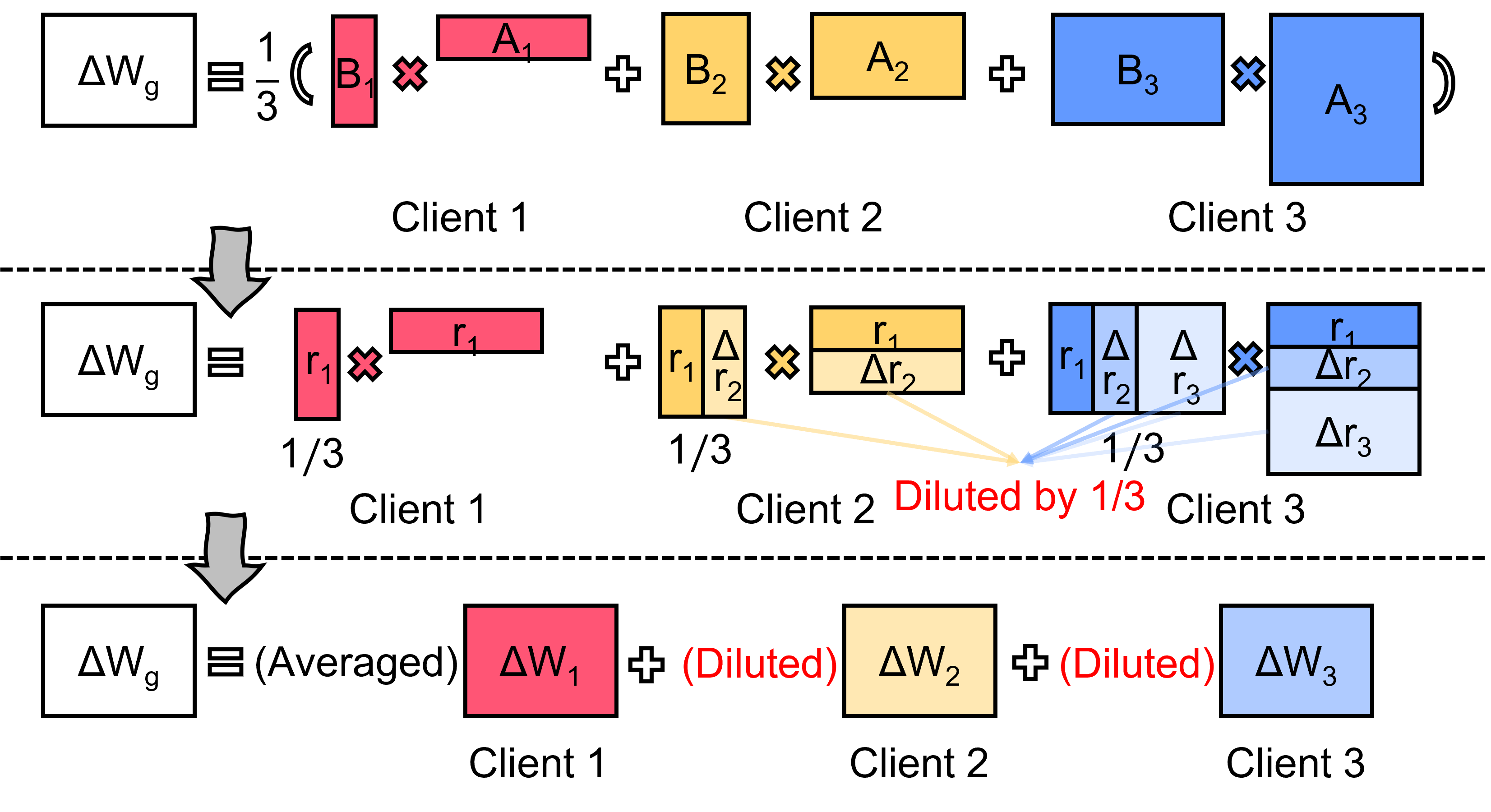}
\caption{Illustration of contributor dilution.}
\label{fig:flexlora}
\vspace{-0.1in}
\end{wrapfigure}
As illustrated in Figure~\ref{fig:flexlora}, different rank partitions have different numbers of effective contributors, whereas conventional FedAvg~\cite{mcmahan2017communication} aggregation normalizes all partitions by the same number of participating clients in each round.
We refer to this as a mismatch between rank-agnostic aggregation weights and rank-dependent client contributions.
Eq.~\ref{eq:contributor_dilution} reveals its direct consequence, namely one-round \emph{contributor dilution}.
The shared partition has $N_1^{(t)}=M$, whereas a higher-rank partition is scaled by $C_h^{(t)}/M<1$ whenever only a subset of sampled clients supports that partition.
Under uniform client sampling without replacement, $\mathbb{E}[C_h^{(t)}/M]=p_{r_h}$, where $p_i$ denotes the corresponding rank coverage defined in Section~\ref{sec:preli}.

\paragraph{Multi-round accumulation.}
To clearly characterize how the one-round dilution accumulates across communication rounds, we introduce a simplifying assumption that isolates the evolution of rank-wise energy and yields a tractable closed-form recursion.

\begin{assumption}[Stable spectral routing]
\label{assump:stable_routing}
For the analytical model, we assume that the retained singular directions remain fixed across rounds without singular-value crossings, and that local training preserves these directions, with each eligible client rescaling direction \(i\) by a gain $\beta_{k,i}^{(t)}$.
\end{assumption}

Let $c_{k,i}=\mathbf{1}\{r_k\ge i\}$ indicate whether client $k$ supports direction $i$. Under Assumption~\ref{assump:stable_routing}, the aggregated gain is $a_i^{(t)}=M^{-1}\sum_{k\in\mathcal{M}_t}c_{k,i}\beta_{k,i}^{(t)}$, giving
$\sigma_i^{(t+1)}=|a_i^{(t)}|\sigma_i^{(t)}$.
Defining
$q_i^{(t)}=\mathbb{E}[(a_i^{(t)})^2]$
and taking expectation over the round-$t$ client sampling yields
$e_i^{(t+1)}=q_i^{(t)}e_i^{(t)}$. For uniform sampling without replacement, $q_i^{(t)}=(1-\tau)p_i^2(\mu_i^{(t)})^2+\tau p_i\nu_i^{(t)}$, where $\tau=(K-M)/(M(K-1))$, and $\mu_i^{(t)}$ and $\nu_i^{(t)}$ are the first and second moments of the local gains among clients supporting direction $i$. Thus, lower contributor coverage reduces the effective energy multiplier, allowing one-round dilution to accumulate across rounds when the retention gap persists, as formalized in Theorem~\ref{thm:cumulative_concentration}. A complete derivation of the multi-round recursion and the proof of the theorem are provided in Appendix~\ref{appendix:multi_round}.

\begin{theorem}[Geometric spectral concentration]
\label{thm:cumulative_concentration}
Under Assumption~\ref{assump:stable_routing}, let
$q_{\mathrm{high}}^{(t)}=\max_{i>r_1}q_i^{(t)}$ and
$q_{\mathrm{shared}}^{(t)}=\min_{i\le r_1}q_i^{(t)}$.
If
$\gamma:=\sup_t q_{\mathrm{high}}^{(t)}/q_{\mathrm{shared}}^{(t)}<1$,
then
\begin{equation}
\label{eq:collapse_bound}
1-\rho_{r_1}^{(T)}
\le
\frac{1-\rho_{r_1}^{(0)}}{\rho_{r_1}^{(0)}}\gamma^T.
\end{equation}
Hence, the higher-rank energy fraction vanishes geometrically as $T\to\infty$, yielding rank collapse.
\end{theorem}

Intuitively, higher-rank directions are supported by fewer clients and thus tend to have smaller effective energy multipliers. Unless stronger local gains offset this lower coverage, the gap compounds across rounds, geometrically concentrating spectral energy in the shared rank $r_1$.

\paragraph{Beyond stable spectral routing.}
Assumption~\ref{assump:stable_routing} is required only for the tractable multi-round recursion across communication rounds. The one-round dilution in Eq.~\ref{eq:contributor_dilution} remains exact without it. In practice, SVD reordering can move spectral components between rank positions with different contributor coverage, while basis drift and cross-direction mixing further redistribute energy across directions. As derived in Appendix~\ref{appendix:beyond_stable_routing}, these effects yield
\begin{equation}
\label{eq:general_energy_dynamics}
e_i^{(t+1)}
=
\mathbb{E}\!\left[(x_{(i)}^{(t)})^2\right]
+
\bar{\xi}_i^{(t)},
\end{equation}
where $x_{(i)}^{(t)}$ is the $i$-th singular magnitude of the reordered aligned component and $\bar{\xi}_i^{(t)}$ captures the resulting perturbation induced by basis drift and cross-direction mixing.
When no crossing occurs and $R_t=0$, we have $x_{(i)}^{(t)}=|a_i^{(t)}|\sigma_i^{(t)}$ and $\bar{\xi}_i^{(t)}=0$, exactly recovering the stable-routing recursion. Thus, spectral reordering and basis drift perturb the exact geometric recursion in Theorem~\ref{thm:cumulative_concentration}, while the coverage-dependent one-round dilution mechanism remains unchanged.

%% file: Sections/4-methodology.tex

The analysis identifies the rank-wise averaging mismatch as the underlying mechanism of rank collapse. Motivated by this insight, we propose a new aggregation strategy to correct this mismatch. We formally introduce raFLoRA, which performs \textbf{r}ank-partitioned \textbf{a}ggregation for \textbf{F}ederated \textbf{Lo}w-\textbf{R}ank \textbf{A}daptation with client heterogeneity.
Unlike conventional FedAvg~\cite{mcmahan2017communication}, which uniformly averages all uploaded updates, raFLoRA partitions the local updates into independent rank-wise components, with aggregation weights aligned to effective rank-wise contributors.

Following the rank partitions defined in Section~\ref{sec:theory}, let $\mathcal{I}_h=\{r_{h-1}+1,\ldots,r_h\}$ for $h=1,\ldots,H$, where $r_0=0$ and $r_H=r_{\max}$. The full rank $r_{\max}$ is thus partitioned into non-overlapping rank partitions. As illustrated in Figure~\ref{fig:raflora}, when three clients have local ranks $r_1$, $r_2$, and $r_3$, these boundaries divide the full rank into three partitions $\mathcal{I}_1$, $\mathcal{I}_2$, and $\mathcal{I}_3$.

\begin{wrapfigure}{r}{0.55\columnwidth}
  \vspace{-0.15in}
  \centering
  \includegraphics[width=0.9\linewidth]{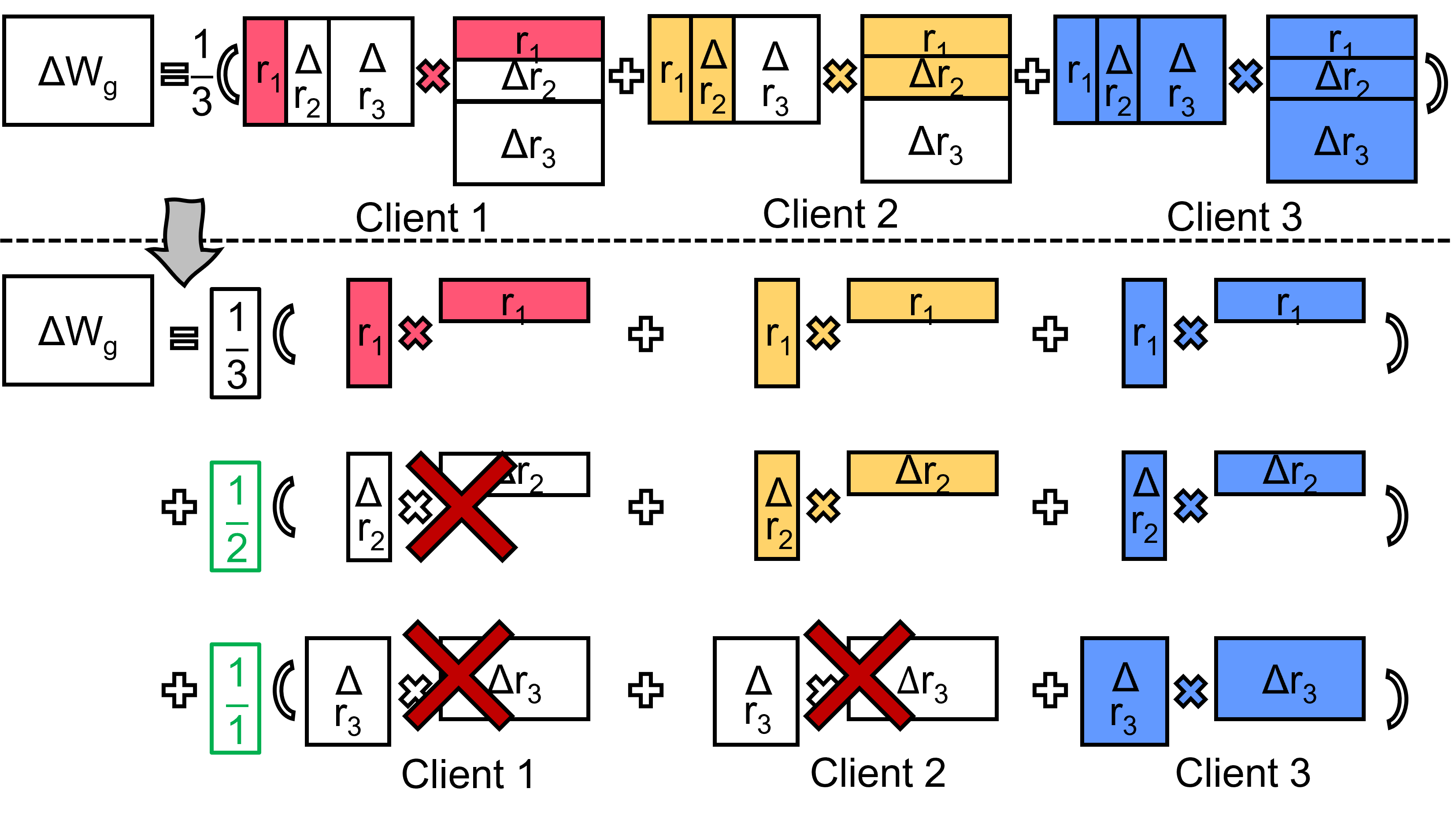}
  \caption{Overview of the rank-partitioned aggregation.}
  \label{fig:raflora}
  \vspace{-0.15in}
\end{wrapfigure}

For each rank partition $\mathcal{I}_h$, aggregation uses only the clients whose local rank satisfies $r_k \ge r_h$, so that global aggregation weights reflect the clients that effectively contribute information at this rank.
Following Section~\ref{sec:theory}, we denote this set of effective contributors by $\mathcal{C}_h^{(t)}=\{k\in\mathcal{M}_t\mid r_k\ge r_h\}$, with the corresponding total sample size $N_h^{(t)}=\sum_{k\in\mathcal{C}_h^{(t)}}n_k$, where $n_k$ denotes the local sample size of client $k$. As illustrated in Figure~\ref{fig:raflora}, all clients contribute to the shared-rank partition $\mathcal{I}_1=\{1,\ldots,r_1\}$, fewer clients contribute to the intermediate partition $\mathcal{I}_2=\{r_1+1,\ldots,r_2\}$, and only the highest-rank client contributes to the last partition $\mathcal{I}_3=\{r_2+1,\ldots,r_3\}$.

Given a specific rank partition $\mathcal{I}_h$, its contribution at round $t$ is defined as
\begin{equation}
\Delta W_{g,h}^{(t+1)}
=
\begin{cases}
\sum_{k \in \mathcal{C}_h^{(t)}}
\frac{n_k}{N_h^{(t)}}
\Big(
B_k^{(t)}[:,\,\mathcal{I}_h]\;
A_k^{(t)}[\mathcal{I}_h,\,:]
\Big), & \mathcal{C}_h^{(t)}\neq\emptyset, \\[4pt]
B_g^{(t)}[:,\,\mathcal{I}_h]\;
A_g^{(t)}[\mathcal{I}_h,\,:], & \mathcal{C}_h^{(t)}=\emptyset.
\end{cases}
\end{equation}
where $[:,\,\mathcal{I}_h]$ selects all rows and the rank indices in $\mathcal{I}_h$ of $B_k^{(t)}$, and $[\mathcal{I}_h,\,:]$ selects the corresponding rank indices and all columns of $A_k^{(t)}$. If no sampled client covers this rank partition in a given round, i.e., $\mathcal{C}_h^{(t)}=\emptyset$, we fall back to the current global LoRA update for this partition rather than skipping it. This fallback prevents previously retained higher-rank information from being discarded. Appendix~\ref{appendix:raf_theory} further derives how this design removes the
coverage-dependent attenuation identified in Section~\ref{sec:theory} while preserving previously retained information when a partition has no contributors.

The global update is then obtained by summing the contributions from all rank partitions, $\Delta W_g^{(t+1)}=\sum_{h=1}^{H}\Delta W_{g,h}^{(t+1)}$. Accordingly, in the example illustrated in Figure~\ref{fig:raflora}, the aggregation weights are $1/3$ for the shared-rank partition where three clients contribute, $1/2$ for the second partition where two clients contribute, and $1$ for the last partition where only a single client provides updates. After aggregation, $\Delta W_g^{(t+1)}$ is projected via SVD to obtain $B_g^{(t+1)}$ and $A_g^{(t+1)}$.

Algorithm~\ref{alg:raflora} summarizes the raFLoRA workflow, which incorporates rank-partitioned aggregation with heterogeneous client ranks. In each round, the server uniformly samples clients, broadcasts the global LoRA updates, and collects local updates in parallel (Lines~3-5). The server then performs rank-partitioned aggregation, aggregating each rank partition using only its effective contributors (Lines~6-11). An SVD then yields the global low-rank updates for the next round (Line~12).

For $L$ LoRA layers, $M$ participating clients, and $H$ rank partitions, reconstructing local updates has the same dominant cost as FlexLoRA, i.e., $O(L\sum_{k=1}^{M}dnr_k)=O(LMdn\bar r)$, where $\bar r$ is the average client rank. The partition-wise aggregation introduces an additional cost of $O(LHMdn)$, leading to a total complexity of
$O(LMdn\bar r(1+H/\bar r))$. Since $H\ll\bar r$ in practice, this extra overhead remains limited. In the next section, we empirically demonstrate the effectiveness of raFLoRA.

\begin{algorithm}[t]
\footnotesize
\caption{Rank-partitioned Aggregation for Federated LoRA with Client Heterogeneity}
\label{alg:raflora}
\begin{algorithmic}[1]
\State \textbf{Input:} Rounds $T$; Clients $K$; participation ratio per round $\rho$; ranks $r_k\in\{r_1,r_2,\dots,r_H\}$ with $r_H=r_{\max}$; $W_{\text{pre}}\in\mathbb{R}^{d\times n}$; initial global LoRA updates $(B_g^{(0)},A_g^{(0)})$ with $B_g^{(0)}\!\in\!\mathbb{R}^{d\times r_{\max}},A_g^{(0)}\!\in\!\mathbb{R}^{r_{\max}\times n}$
\For{$t=0$ \textbf{to} $T-1$}
    \State Sample participating clients $\mathcal{M}_t \subseteq \{1,\dots,K\}$ uniformly with $|\mathcal{M}_t|=M$

    \State Only broadcast LoRA updates for each client $k\in\mathcal{M}_t$:
    \quad $\widetilde{B}_k^{(t)} = B_{g}^{(t)}[:,:r_k],\widetilde{A}_k^{(t)} = A_{g}^{(t)}[:r_k,:]$
    \State \textbf{In parallel}, each client $k\in\mathcal{M}_t$ trains from 
    $(\widetilde{B}_k^{(t)},\widetilde{A}_k^{(t)})$ and uploads $(B_k^{(t)},A_k^{(t)})$

    \State \textbf{Rank-partitioned Aggregation} at Server: \quad Initialization $\Delta W_g^{(t+1)} = 0$
    \For{$h=1$ \textbf{to} $H$}
        \State $\mathcal{C}_h^{(t)}\gets\{k\in\mathcal{M}_t:r_k\ge r_h\}$,
        \quad
        $N_h^{(t)}=\sum_{j\in\mathcal{C}_h^{(t)}}n_j$

        \State $\Delta W_{g,h}^{(t+1)}\gets
        \begin{cases}
        \sum_{k\in\mathcal{C}_h^{(t)}}
        \frac{n_k}{N_h^{(t)}}
        \big(
        B_k^{(t)}[:,\mathcal{I}_h]
        A_k^{(t)}[\mathcal{I}_h,:]
        \big), &
        \mathcal{C}_h^{(t)}\neq\emptyset,\\
        B_g^{(t)}[:,\mathcal{I}_h]
        A_g^{(t)}[\mathcal{I}_h,:], &
        \mathcal{C}_h^{(t)}=\emptyset.
        \end{cases}$

        \State $\Delta W_g^{(t+1)}
        \gets
        \Delta W_g^{(t+1)}
        +
        \Delta W_{g,h}^{(t+1)}$
    \EndFor

    \State Decompose $\Delta W_g^{(t+1)}$ to $(B_g^{(t+1)},A_g^{(t+1)})$ via SVD
\EndFor
\State \textbf{return} $W_{\text{pre}} + B_g^{(T)}A_g^{(T)}$
\end{algorithmic}
\end{algorithm}

%% file: Sections/5-experiments.tex
In this section, we present the experimental setup, then we evaluate raFLoRA in terms of accuracy, rank collapse prevention, efficiency, sensitivity \& robustness, and extended settings.

\subsection{Experimental Setup}
\label{exp-setup}
The experimental setup specifies the models and datasets for each task, metrics for evaluation, non-IID data partitioning, baselines, and hyperparameter settings. All experiments are performed on a single GPU, using an NVIDIA RTX 5090 (32 GB) or an NVIDIA H100-SXM (80 GB). Additional implementation details and more hyperparameter settings are provided in Appendix~\ref{appendix:hyperparameters}.

\textbf{Models and Datasets.}
We evaluate encoder-only and decoder-only models across image classification, text classification, mathematical reasoning, and commonsense reasoning.
We use ViT-base~\cite{dosovitskiy2021vit} on \mbox{CIFAR100}~\cite{Krizhevsky2009cifar100}, RoBERTa-base~\cite{liu2019roberta} on 20 Newsgroups~\cite{lang199520ng}, and LLaMA-3.2-3B~\cite{meta2024llama3.2} and LLaMA-3.1-8B~\cite{grattafiori2024llama3} on GSM8K~\cite{cobbe2021gsm8k}.
For commonsense reasoning, the two LLaMA models are fine-tuned on Commonsense15K~\cite{hu2023commonsense} and evaluated on eight benchmarks.

\textbf{Metrics.}
We report test accuracy as mean $\pm$ standard deviation over three random seeds and the energy ratio. Communication cost is measured by the total upload and download volume per client per round, while computational cost is measured by total training runtime.

\textbf{Data Partitioning.}
We consider IID and multiple non-IID partitions with the following default settings. For vision and NLU tasks, we use Dirichlet partitioning, where smaller $\alpha$ indicates stronger heterogeneity, and pathological partitioning~\cite{zhang2025fedhello}, where each client accesses a subset of labels. For example, c20($\alpha=1$) denotes 20 labels per client with $\alpha=1$. GSM8K is uniformly partitioned, while Commonsense15K is partitioned by answer type with $\alpha=0.5$.

\textbf{Baselines.}
We compare against state-of-the-art heterogeneous-rank FedLoRA methods.
HetLoRA~\cite{cho2024hetlora} aligns heterogeneous ranks through zero padding.
FLoRA~\cite{wang2024flora} aggregates heterogeneous updates by stacking and broadcasts the stacked updates to clients.
FlexLoRA~\cite{bai2024flexlora} uses bias-free aggregation and allocates heterogeneous ranks via SVD. For a fair comparison, the same fallback mechanism is applied to all baselines.

\textbf{Hyperparameters.}
Unless otherwise specified, we use 100 clients with 10\% participation per round, AdamW, and one local epoch. The learning rate starts at $5\times10^{-4}$ and decays linearly. LoRA ranks are uniformly sampled from $\{8,16,32,48,64\}$, with LoRA alpha set to $r_k$ for unit scaling. Vision and NLU tasks run for 100 rounds with LoRA applied to all linear layers, while reasoning tasks run for 20 rounds with LoRA applied only to the $Q$ and $V$ modules.

\begin{table}[t]
\vspace{-0.3in}
\caption{Accuracy comparison across diverse tasks. Accuracy is reported in (\%), with the best results highlighted in bold. Results are obtained using ViT-base for vision, RoBERTa-base for text, LLaMA-3.2-3B/LLaMA-3.1-8B for mathematical and commonsense reasoning.}
\label{tab:main-results}
\centering
\scriptsize
\resizebox{\linewidth}{!}{
\begin{tabular}{l|c|c|c|c}
\toprule
\multirow{2}{*}{\textbf{Methods}} & Vision & Language & {Mathematical Reasoning} & {Commonsense Reasoning}\\
 & \textbf{CIFAR100}  & \textbf{20NG} & \textbf{GSM8K} & \textbf{Commonsense15K}\\
\midrule
HetLoRA & 84.04$_{\pm1.37}$ & 62.06$_{\pm0.67}$ & 36.75$_{\pm1.23}$/56.43$_{\pm0.29}$ & 71.89$_{\pm0.41}$/79.45$_{\pm1.00}$ \\
FLoRA & 86.30$_{\pm0.79}$ & 61.91$_{\pm0.52}$ & 36.09$_{\pm0.57}$/56.25$_{\pm0.85}$ & 71.02$_{\pm0.18}$/79.01$_{\pm0.38}$ \\
FlexLoRA & 84.02$_{\pm1.17}$ & 63.05$_{\pm0.83}$ & 40.21$_{\pm0.42}$/58.43$_{\pm0.74}$ & 74.33$_{\pm0.50}$/80.88$_{\pm0.67}$ \\
\rowcolor{lightorange}
raFLoRA & \textbf{86.59}$_{\pm0.75}$ & \textbf{64.80}$_{\pm0.34}$ & \textbf{41.72}$_{\pm0.70}$/\textbf{59.06}$_{\pm0.27}$ & \textbf{74.86}$_{\pm0.39}$/\textbf{81.15}$_{\pm0.08}$\\
\bottomrule
\end{tabular}
}
\vspace{-0.1in}
\end{table}

\begin{figure}[h]
\centering
\begin{subfigure}{0.48\linewidth}
\centering
\includegraphics[width=\linewidth]{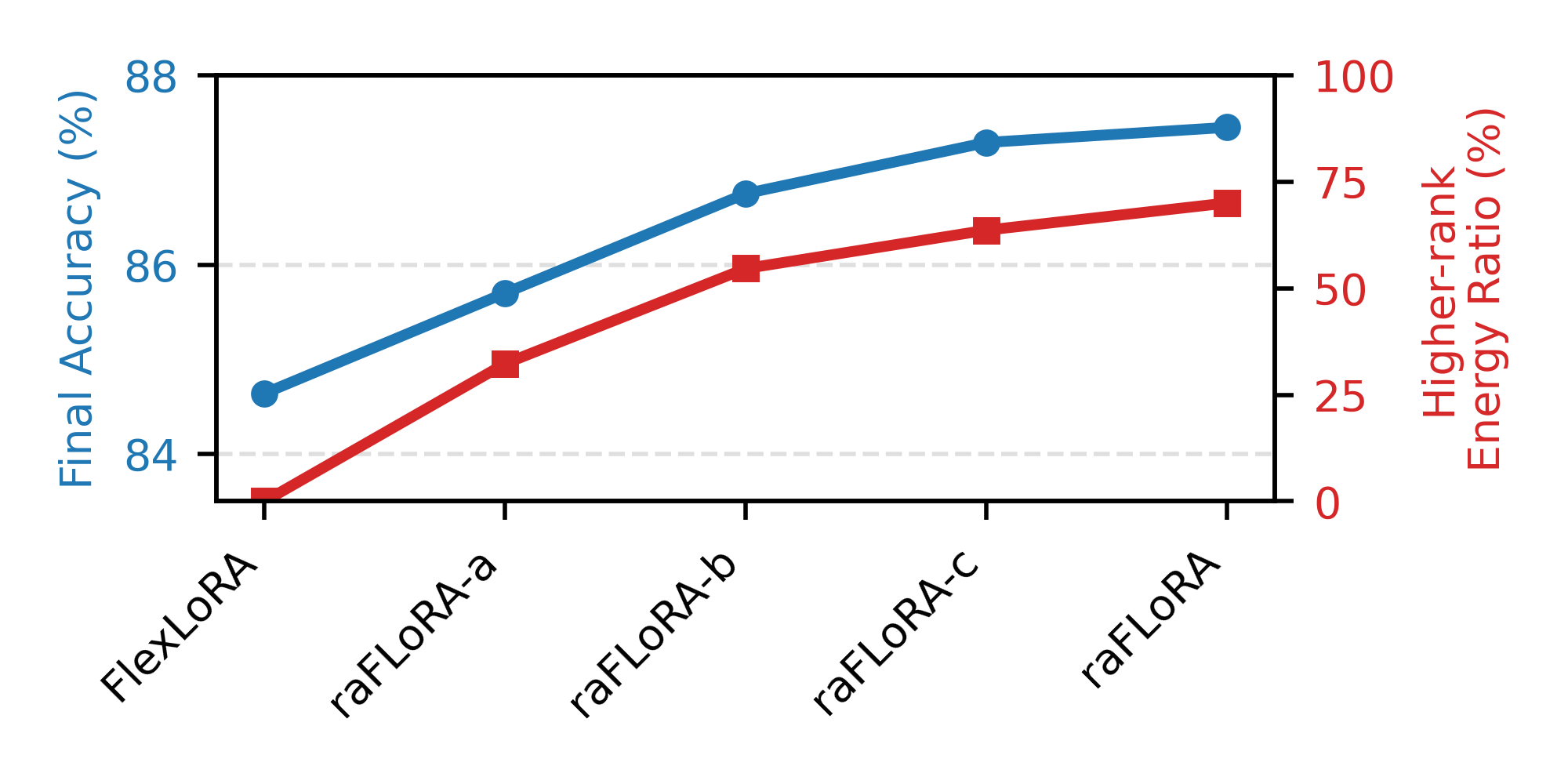}
\caption{Rank collapse prevention}
\label{fig:rank-collapse-prevention}
\end{subfigure}
\begin{subfigure}{0.48\linewidth}
\centering
\includegraphics[width=\linewidth]{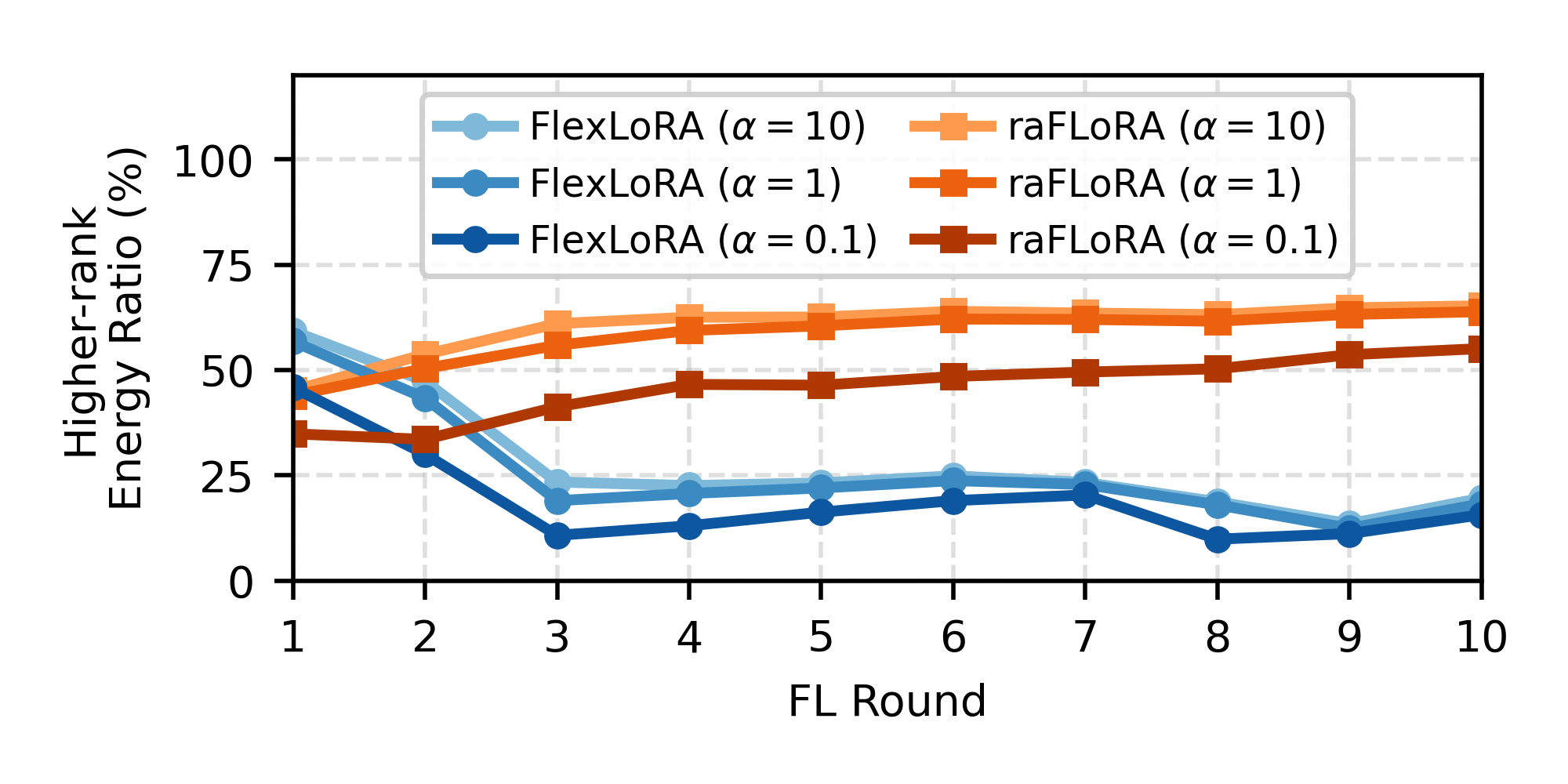}
\caption{Energy under data heterogeneity}
\label{fig:energy-niid}
\end{subfigure}

\caption{Rank collapse prevention and higher-rank energy ratio dynamics.}
\label{fig:rank-collapse-analysis}
\end{figure}

\subsection{Accuracy across Diverse Tasks}
We evaluate raFLoRA across vision, language, and reasoning tasks, with data partitioning and hyperparameter details provided in Appendix~\ref{appendix:hyperparameters}.
As shown in Table~\ref{tab:main-results}, \textbf{raFLoRA delivers consistent improvements over baselines across a wide range of tasks.}
Specifically, raFLoRA outperforms FlexLoRA and HetLoRA by about $2.6\%$ on CIFAR100, and improves over FLoRA by nearly $3.0\%$ on 20NG.
On more challenging reasoning tasks, raFLoRA also achieves higher performance on GSM8K with both LLaMA-3.2-3B and LLaMA-3.1-8B models.
For commonsense reasoning, raFLoRA achieves the highest average accuracy across eight sub-datasets, with detailed results in Appendix~\ref{appendix:commonsense-results}. FLoRA's weaker performance on some tasks may stem from its cold-start behavior, where LoRA parameters are reinitialized each round. Additional training dynamics and generality evaluations across other heterogeneous FedLoRA methods are provided in Appendices~\ref{appendix:training-dynamics} and~\ref{appendix:generality}, respectively.

\subsection{Rank Collapse Prevention}
We construct partial variants (raFLoRA-a/b/c) on CIFAR100, applying effective-contributor weighting up to rank partitions $\{8,16\}$, $\{8,16,32\}$, and $\{8,16,32,48\}$, respectively, while retaining baseline aggregation for the remaining partitions. \textbf{By preserving higher-rank energy, rank-partitioned aggregation prevents rank collapse and improves final performance.}
As shown in Figure~\ref{fig:rank-collapse-prevention}, applying rank-partitioned aggregation to more partitions improves higher-rank energy preservation and is accompanied by better accuracy. A post-SVD alternative is evaluated in Appendix~\ref{appendix:post_svd_scaling}.

Additionally, we analyze the higher-rank energy ratio dynamics of FlexLoRA and raFLoRA under varying data heterogeneity on CIFAR100. As shown in Figure~\ref{fig:energy-niid}, stronger data heterogeneity reduces the preservation of higher-rank energy. This observation supports our analysis in Section~\ref{sec:analysis}, indicating that data heterogeneity weakens the alignment and effective accumulation of higher-rank updates during aggregation.

\subsection{Communication and Computation Costs}
\label{subsec-cost}
\begin{wraptable}{r}{0.5\columnwidth}
\vspace{-0.2in}
\caption{Comparison of total training runtime and average communication cost per client per round.}
\label{tab:cost}
\centering
\scriptsize
\resizebox{\linewidth}{!}{
\begin{tabular}{l|cc|cc}
\toprule
\multirow{2}{*}{\textbf{Methods}} & \multicolumn{2}{c|}{ViT-base} & \multicolumn{2}{c}{LLaMA-3.1-8B}\\
 & \textbf{Comp.} & \textbf{Comm.} & \textbf{Comp.} & \textbf{Comm.}\\
\midrule
HetLoRA        & 40m50s & 41MB  & 1h33m55s & 109MB \\
FLoRA          & 47m05s & 226MB & 1h40m11s & 580MB\\
FlexLoRA       & 40m40s & 41MB  & 1h38m09s & 109MB\\
\rowcolor{lightorange}
raFLoRA        & 42m55s & 41MB  & 1h50m56s & 109MB \\
\bottomrule
\end{tabular}
}
\vspace{-0.1in}
\end{wraptable}
We evaluate the communication and computational costs of different methods during fine-tuning. As shown in Table~\ref{tab:cost}, \textbf{raFLoRA maintains competitive communication efficiency while introducing only modest additional computational overhead}. Compared with FLoRA, raFLoRA reduces the communication cost to about $18\%$ on both ViT-base and LLaMA-3.1-8B, as it avoids synchronizing the full global update. Although rank-partitioned aggregation slightly increases runtime, the overhead remains controlled and is consistent with the complexity analysis in Section~\ref{sec:method}.

\subsection{Sensitivity and Robustness Analyses}
\label{sen-rob}
Since rank collapse increases sensitivity to rank configurations and data heterogeneity, we compare raFLoRA and the strongest baseline FlexLoRA across diverse heterogeneous settings.

\textbf{Effect of non-IID settings.}
We evaluate CIFAR100 and 20NG under varying data heterogeneity. Figures~\ref{fig:acc-niid} and~\ref{fig:ablation-niid} show that raFLoRA is more robust than FlexLoRA. As heterogeneity increases, raFLoRA mitigates this degradation and consistently achieves higher accuracy.

\begin{figure}[t]
  \vspace{-0.3in}
  \centering
  \begin{subfigure}{0.245\linewidth}
    \centering
    \includegraphics[width=\linewidth]{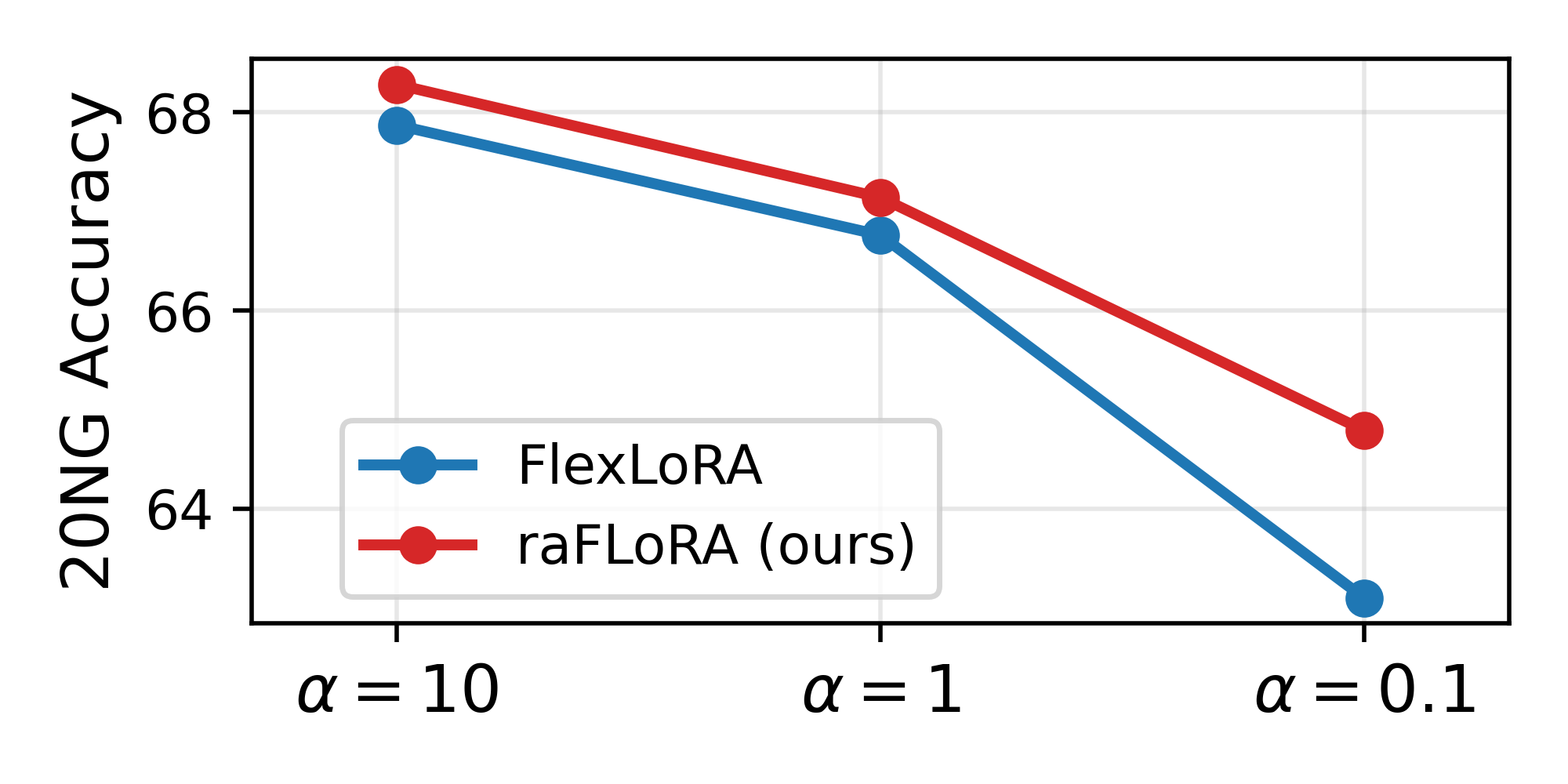}
    \caption{non-IID}
    \label{fig:ablation-niid}
  \end{subfigure}
  \hfill
  \begin{subfigure}{0.245\linewidth}
    \centering
    \includegraphics[width=\linewidth]{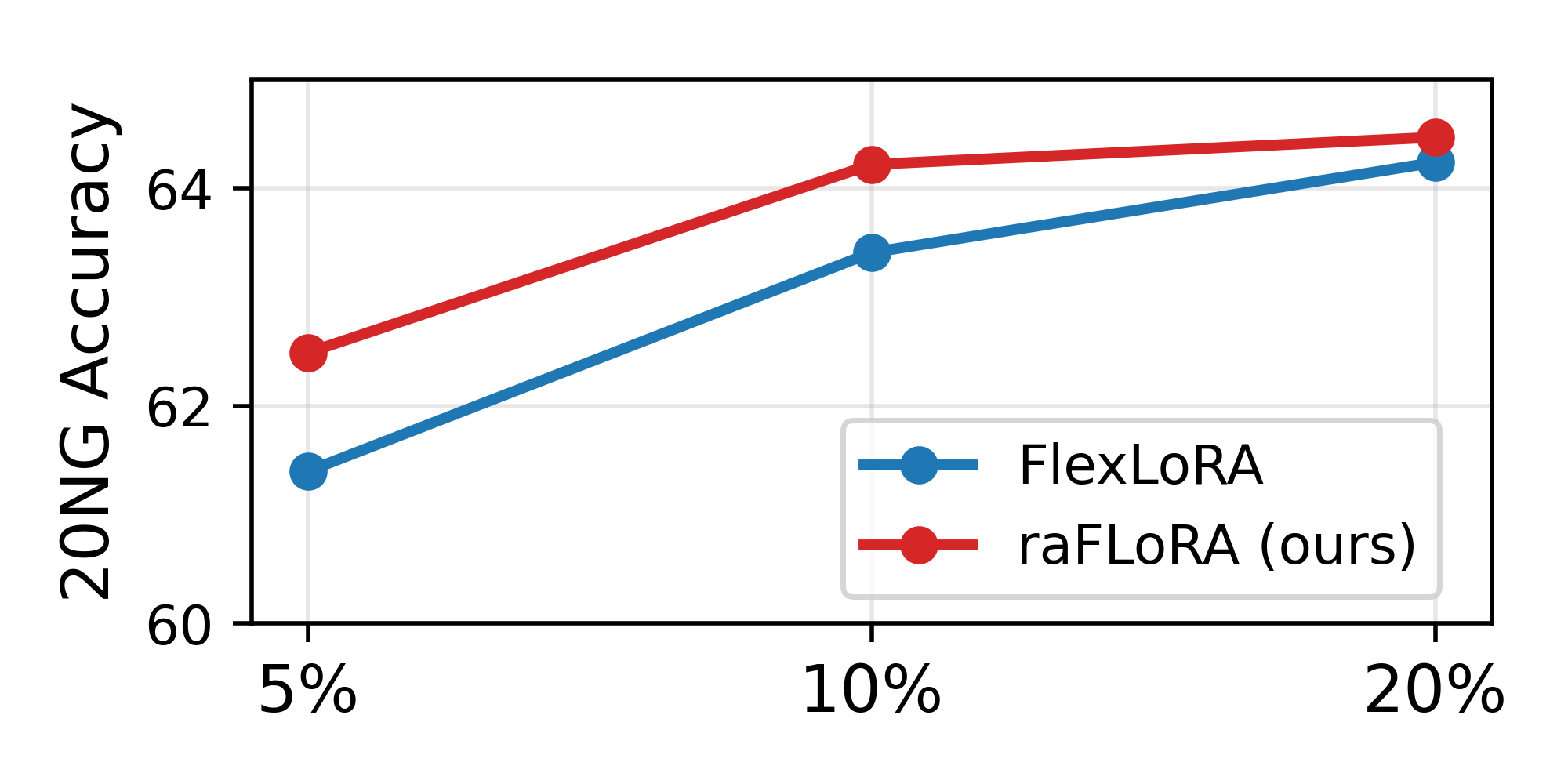}
    \caption{Client Participation}
    \label{fig:ablation-cpr}
  \end{subfigure}
  \hfill
  \begin{subfigure}{0.245\linewidth}
    \centering
    \includegraphics[width=\linewidth]{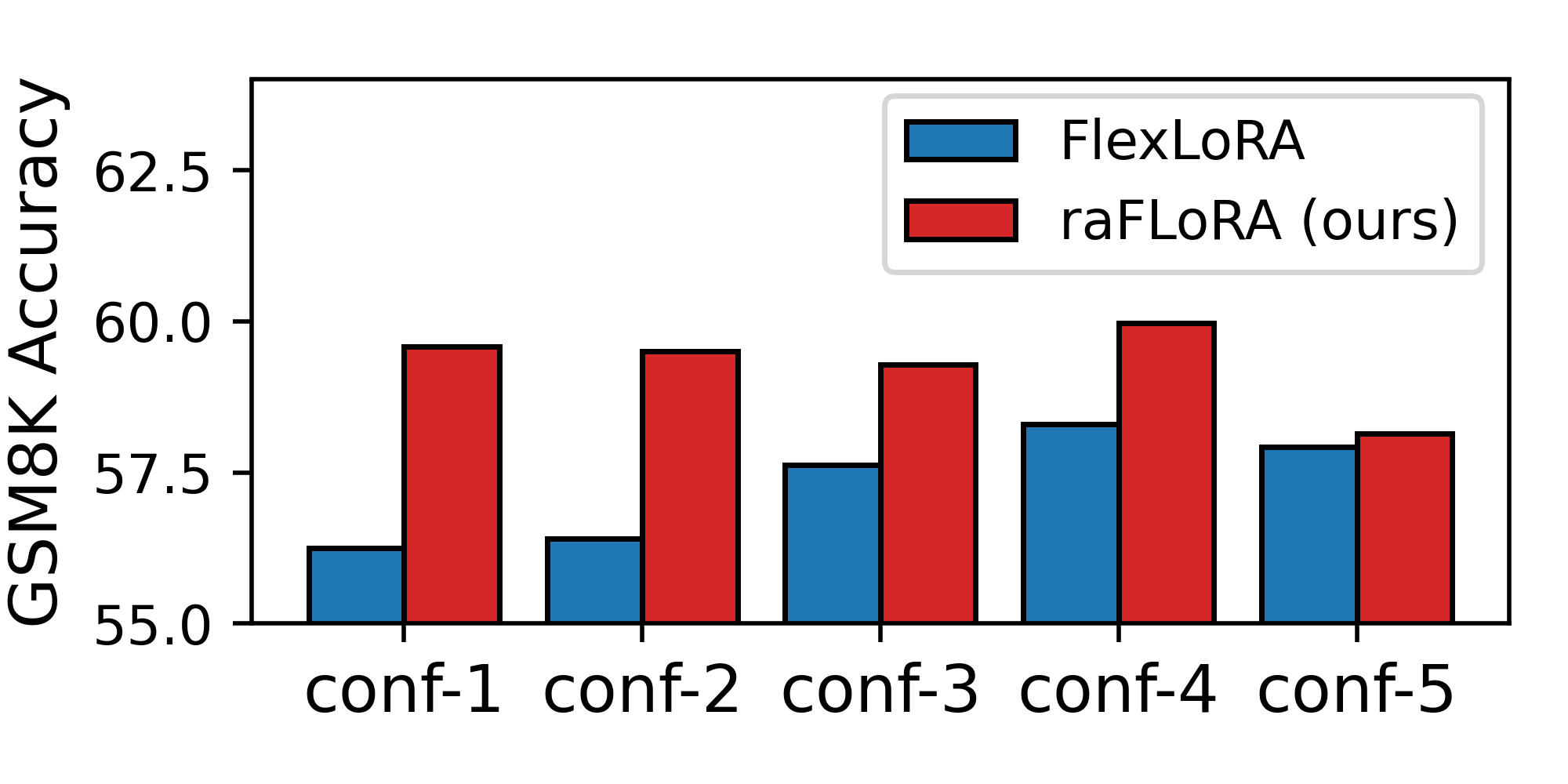}
    \caption{Rank-configs}
    \label{fig:ablation-ranks}
  \end{subfigure}
  \hfill
  \begin{subfigure}{0.245\linewidth}
    \centering
    \includegraphics[width=\linewidth]{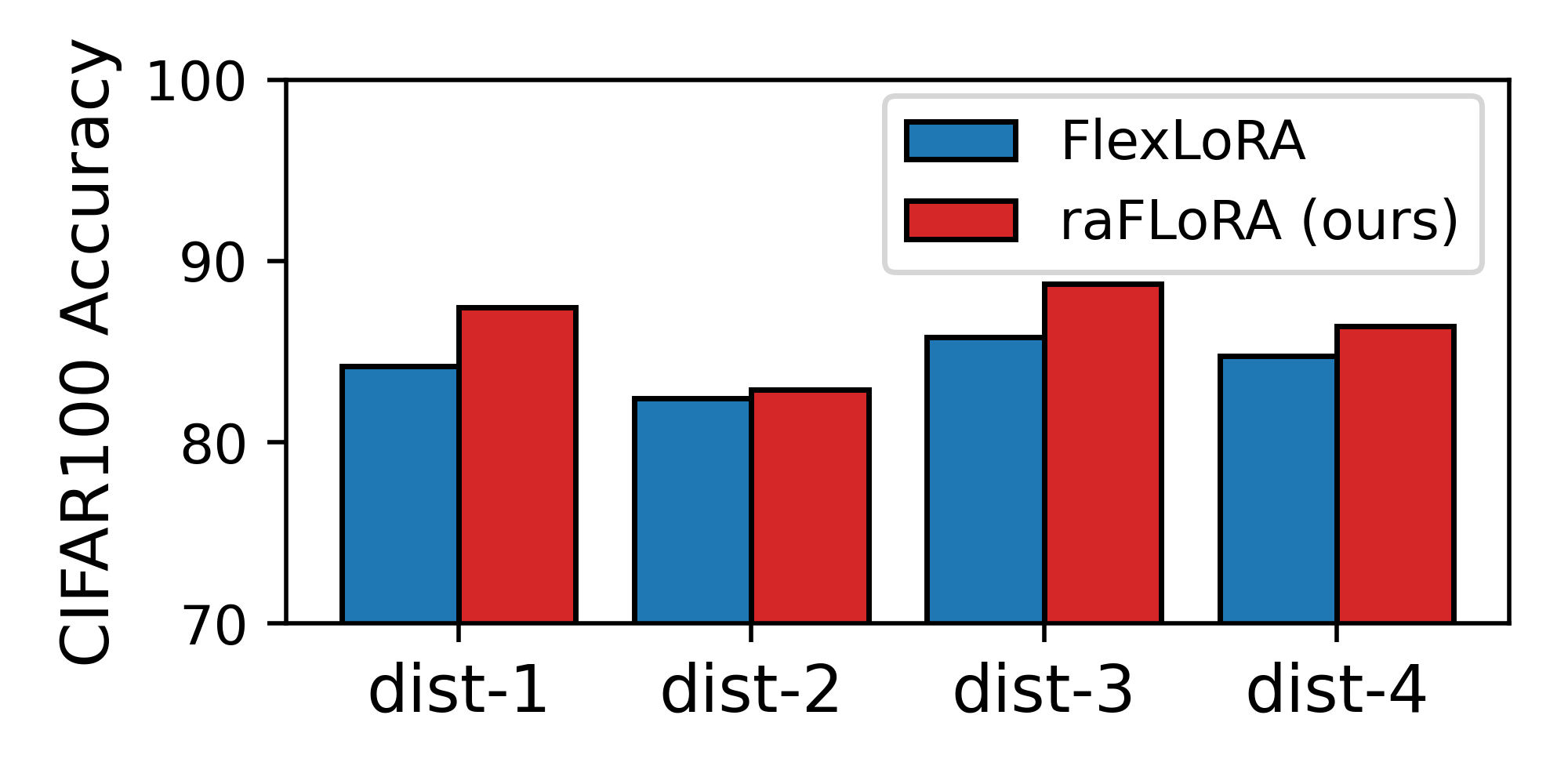}
    \caption{Distributions}
    \label{fig:ablation-distributions}
  \end{subfigure}

  \caption{Sensitivity and robustness analyses of FlexLoRA and raFLoRA under different settings. The detailed configurations for (c) and (d) are provided in Section~\ref{sen-rob}.}
  \label{fig:ablation}
  \vspace{-0.2in}
\end{figure}

\textbf{Effect of client participation rates.}
We evaluate 20NG under varying client participation rates.
Figure~\ref{fig:ablation-cpr} shows that raFLoRA consistently outperforms FlexLoRA, with accuracy first increasing and then stabilizing as participation increases.
With only one participating client, raFLoRA reduces to FlexLoRA because no rank-wise dilution occurs. The reliability of sparsely supported higher-rank partitions is further discussed in Appendix~\ref{sec:limitations}.

\textbf{Effect of rank configurations.}
We evaluate CIFAR100 and GSM8K under five rank configurations: conf-1 to conf-5 correspond to $\{1,16,32,48,64\}$, $\{4,16,32,48,64\}$, $\{8,16,32,48,64\}$, $\{8,16,32,48,96\}$, and $\{8,16,32,48,128\}$, respectively. Figures~\ref{fig:acc-config} and~\ref{fig:ablation-ranks} show that raFLoRA consistently outperforms FlexLoRA and avoids its pronounced sensitivity to the shared rank. Additional rank configurations and LoRA module insertion settings are evaluated in Appendix~\ref{appendix:extended-exp}.

\textbf{Effect of rank distributions.}
We evaluate CIFAR100 with ranks $\{8,16,32,48,64\}$ under four distributions: uniform (dist-1: $\{0.2,0.2,0.2,0.2,0.2\}$), low-rank skewed (dist-2: $\{0.7,0.1,0.1,0.05,0.05\}$), high-rank skewed (dist-3: $\{0.05,0.05,0.1,0.1,0.7\}$), and bell-shaped (dist-4: $\{0.05,0.1,0.7,0.1,0.05\}$). Figure~\ref{fig:ablation-distributions} shows that raFLoRA outperforms FlexLoRA under most distributions and remains comparable under dist-2. This suggests that its benefit in preserving the hierarchical rank-energy structure diminishes when most clients use low ranks.

\begin{table}[h]
\vspace{-0.05in}
\centering
\small
\begin{minipage}[t]{0.52\linewidth}
\centering
\caption{Performance under Gaussian noise.}
\label{tab:noise-low-rank}
\resizebox{\linewidth}{!}{
\begin{tabular}{lcccc}
\toprule
\textbf{Methods} & \textbf{$\nu=0.0$} & \textbf{$\nu=0.1$} & \textbf{$\nu=0.3$} & \textbf{$\nu=0.5$} \\
\midrule
FlexLoRA & 83.94\% & 85.15\% & 82.67\% & 81.45\% \\
raFLoRA  & \textbf{87.43\%} & \textbf{87.47\%} & \textbf{86.91\%} & \textbf{86.61\%} \\
\bottomrule
\end{tabular}
}
\end{minipage}
\hspace{0.02\linewidth}
\begin{minipage}[t]{0.44\linewidth}
\centering
\caption{Extension to PEFT variants.}
\label{tab:peft-extension}
\resizebox{\linewidth}{!}{
\begin{tabular}{lccc}
\toprule
\textbf{Methods} & QLoRA & AdaLoRA & DoRA \\
\midrule
FlexLoRA & 85.66\% & 87.34\% & 73.88\% \\
raFLoRA  & \textbf{86.53\%} & \textbf{88.73\%} & \textbf{86.63\%} \\
\bottomrule
\end{tabular}
}
\end{minipage}
\vspace{-0.15in}
\end{table}

\subsection{Extended Experiments}
\label{ext-exp}

\textbf{Extension to noisy low-rank clients.}
To evaluate the case where low-rank clients (rank=8) also have lower-quality data sources, we inject zero-mean Gaussian noise $\epsilon \sim \mathcal{N}(0,\nu^2)$ into their data on CIFAR100, where $\nu$ denotes the noise standard deviation. As shown in Table~\ref{tab:noise-low-rank}, raFLoRA consistently outperforms FlexLoRA as the noise level increases, suggesting robustness to noisy low-rank clients.

\textbf{Extension to LoRA variants.}
We extend raFLoRA to QLoRA~\cite{dettmers2023qlora}, AdaLoRA~\cite{zhang2023adalora}, and DoRA~\cite{liu2024dora}. Table~\ref{tab:peft-extension} shows accuracy gains across variants on CIFAR100. The degradation of FlexLoRA-DoRA suggests that DoRA is sensitive to rank collapse, since magnitude reweighting cannot recover attenuated directional information, whereas raFLoRA avoids this issue by rank-partitioned aggregation.

%% file: Sections/appendix.tex
\section{Exact One-Round Contributor Dilution}
\label{appendix:one_round_dilution}

We derive Eq.~\ref{eq:contributor_dilution} one-round contributor dilution. For client $k$, define its update in partition $h$ as
\[
\Delta W_{k,h}^{(t)}
=
B_k^{(t)}[:,\mathcal{I}_h]
A_k^{(t)}[\mathcal{I}_h,:].
\]
Since client $k$ contains exactly the partitions satisfying $r_k\ge r_h$,
\[
\Delta W_k^{(t)}
=
\sum_{h:r_k\ge r_h}
\Delta W_{k,h}^{(t)}.
\]
Substituting this decomposition into FedAvg gives
\[
\Delta W_g^{(t+1)}
=
\frac{1}{M}
\sum_{k\in\mathcal{M}_t}
\sum_{h:r_k\ge r_h}
\Delta W_{k,h}^{(t)}
\nonumber\\
=
\sum_{h=1}^{H}
\frac{1}{M}
\sum_{k\in\mathcal{C}_h^{(t)}}
\Delta W_{k,h}^{(t)}
\nonumber\\
=
\sum_{h=1}^{H}
\frac{C_h^{(t)}}{M}
\overline{\Delta W}_h^{(t)},
\]
which proves Eq.~\ref{eq:contributor_dilution}.

Under uniform sampling without replacement,
$C_h^{(t)}\sim\mathrm{Hypergeo}(K,Kp_{r_h},M)$, and hence
\[
\mathbb{E}\left[\frac{C_h^{(t)}}{M}\right] = p_{r_h}.
\]
Therefore, higher-rank partitions receive systematically smaller normalization factors whenever their contributor coverage is lower.
\section{Multi-Round Dynamics under Stable Spectral Routing}
\label{appendix:multi_round}
Under Assumption~\ref{assump:stable_routing}, direction $i$ preserves its spectral identity across rounds. We treat the round-wise client gains as fixed with respect to the sampling history, with randomness arising from uniform client sampling. Recall
\[
a_i^{(t)}
=
\frac{1}{M}
\sum_{k\in\mathcal{M}_t}
c_{k,i}\beta_{k,i}^{(t)},
\qquad
\sigma_i^{(t+1)}
=
|a_i^{(t)}|\sigma_i^{(t)}.
\]

Let
\[
\mu_i^{(t)}
=
\mathbb{E}[\beta_{k,i}^{(t)}\mid r_k\ge i],
\qquad
\nu_i^{(t)}
=
\mathbb{E}[(\beta_{k,i}^{(t)})^2\mid r_k\ge i].
\]
For $X_{k,i}^{(t)}=c_{k,i}\beta_{k,i}^{(t)}$, rank coverage gives
\[
\mathbb{E}[X_{k,i}^{(t)}]
=
p_i\mu_i^{(t)},
\qquad
\mathbb{E}[(X_{k,i}^{(t)})^2]
=
p_i\nu_i^{(t)}.
\]
Since $a_i^{(t)}$ is the mean of $M$ clients sampled uniformly without
replacement, letting
\[
\tau=\frac{K-M}{M(K-1)}
\]
gives
\begin{equation}
\label{eq:appendix_q}
q_i^{(t)}
:=
\mathbb{E}[(a_i^{(t)})^2]
=
(1-\tau)p_i^2(\mu_i^{(t)})^2
+
\tau p_i\nu_i^{(t)}.
\end{equation}
Thus,
\begin{equation}
\label{eq:appendix_energy_recursion}
e_i^{(t+1)}
=
q_i^{(t)}e_i^{(t)}.
\end{equation}

We now prove Theorem~\ref{thm:cumulative_concentration}.
Define
\[
S_t=\sum_{i\le r_1}e_i^{(t)},
\qquad
H_t=\sum_{i>r_1}e_i^{(t)}.
\]
Using Eq.~\ref{eq:appendix_energy_recursion},
\[
H_{t+1}
\le
q_{\mathrm{high}}^{(t)}H_t,
\qquad
S_{t+1}
\ge
q_{\mathrm{shared}}^{(t)}S_t.
\]
Therefore,
\[
\frac{H_T}{S_T}
\le
\frac{H_0}{S_0}
\prod_{t=0}^{T-1}
\frac{q_{\mathrm{high}}^{(t)}}
     {q_{\mathrm{shared}}^{(t)}}
\le
\frac{H_0}{S_0}\gamma^T.
\]
Since
\[
1-\rho_{r_1}^{(T)}
=
\frac{H_T}{S_T+H_T}
\le
\frac{H_T}{S_T}
\]
and
\[
\frac{H_0}{S_0}
=
\frac{1-\rho_{r_1}^{(0)}}{\rho_{r_1}^{(0)}},
\]
we obtain
\[
1-\rho_{r_1}^{(T)}
\le
\frac{1-\rho_{r_1}^{(0)}}{\rho_{r_1}^{(0)}}
\gamma^T,
\]
which proves Theorem~\ref{thm:cumulative_concentration}.

Equation~\ref{eq:appendix_q} separates contributor coverage from local training: lower $p_i$ suppresses a direction, while stronger local gains can compensate. Thus, rank collapse does not require weaker higher-rank updates, but occurs when their lower coverage is insufficiently offset.
\section{Analysis Beyond Stable Spectral Routing}
\label{appendix:beyond_stable_routing}

We next examine the dynamics when Assumption~\ref{assump:stable_routing} does not hold.
The one-round contributor dilution in Eq.~\ref{eq:contributor_dilution} remains exact, while spectral reordering and basis drift complicate the multi-round evolution.

Let
$\Delta W_g^{(t)}=U^{(t)}\Sigma^{(t)}V^{(t)\top}$.
We decompose each selected client's update as
\[
\Delta W_k^{(t)}
=
U^{(t)}D_k^{(t)}\Sigma^{(t)}V^{(t)\top}
+
P_k^{(t)},
\]
where
$D_k^{(t)}
=\operatorname{diag}
(c_{k,1}\beta_{k,1}^{(t)},\ldots,
c_{k,r_{\max}}\beta_{k,r_{\max}}^{(t)})$
captures evolution aligned with the current singular basis, and $P_k^{(t)}$ captures basis drift and cross-direction mixing. Averaging over the selected clients gives
\[
\Delta W_g^{(t+1)}
=
U^{(t)}D^{(t)}\Sigma^{(t)}V^{(t)\top}
+
R_t,
\label{eq:general_aggregate_appendix}
\]
where
$D^{(t)}=M^{-1}\sum_{k\in\mathcal{M}_t}D_k^{(t)}$
and
$R_t=M^{-1}\sum_{k\in\mathcal{M}_t}P_k^{(t)}$.

The $i$-th diagonal entry of $D^{(t)}$ is $a_i^{(t)}=M^{-1}\sum_{k\in\mathcal{M}_t}
c_{k,i}\beta_{k,i}^{(t)}$. Hence, the aligned component has singular magnitudes $x_i^{(t)}=|a_i^{(t)}|\sigma_i^{(t)}$. Let $x_{(1)}^{(t)}\ge\cdots\ge x_{(r_{\max})}^{(t)}$
denote these values after sorting. Thus, singular-value crossings can move a spectral component to a different rank position and expose it to a different contributor coverage in the next
round.

Let $\epsilon_t=\|R_t\|_2$.
By the singular-value perturbation bound,
\[
\left|
\sigma_i^{(t+1)}
-
x_{(i)}^{(t)}
\right|
\le
\epsilon_t.
\label{eq:weyl_bound}
\]
It follows that
\[
(\sigma_i^{(t+1)})^2
=
(x_{(i)}^{(t)})^2
+
\xi_i^{(t)},
\]
where
\[
|\xi_i^{(t)}|
\le
2\epsilon_t x_{(i)}^{(t)}
+
\epsilon_t^2.
\]
Taking expectation yields
\begin{equation}
e_i^{(t+1)}
=
\mathbb{E}\!\left[(x_{(i)}^{(t)})^2\right]
+
\bar{\xi}_i^{(t)},
\label{eq:general_expected_energy}
\end{equation}
with
$|\bar{\xi}_i^{(t)}|
\le
\mathbb{E}[2\epsilon_t x_{(i)}^{(t)}+\epsilon_t^2]$.

Unlike under stable spectral routing, $x_{(i)}^{(t)}$ need not correspond to direction $i$ from the previous round. Therefore, once singular-value crossings or basis drift occur, the multi-round dynamics generally cannot be reduced to a direction-wise multiplicative recursion of the form $e_i^{(t+1)}=q_i^{(t)}e_i^{(t)}$.

When there is no singular-value crossing and $R_t=0$, however, $x_{(i)}^{(t)}=|a_i^{(t)}|\sigma_i^{(t)}$. Under the sampling conditions used in Appendix~\ref{appendix:multi_round}, Eq.~\ref{eq:general_expected_energy} then reduces to $e_i^{(t+1)}=q_i^{(t)}e_i^{(t)}$.

This decomposition clarifies the role of SVD beyond the simplified model. Rank heterogeneity creates the coverage imbalance, while SVD determines which spectral components occupy rank positions with different coverage. Reordering can therefore move a component between coverage levels, and basis drift can transfer energy across directions. These effects may alter the exact multi-round trajectory and invalidate the geometric rate in Theorem~\ref{thm:cumulative_concentration}, but they do not change the one-round contributor dilution in Eq.~\ref{eq:contributor_dilution}.
\section{Why raFLoRA Mitigates Rank Collapse}
\label{appendix:raf_theory}

We next connect the proposed aggregation rule to the contributor-dilution mechanism identified above. For a nonempty partition $h$, conventional aggregation gives
\[
\Delta W_{g,h,\mathrm{base}}^{(t+1)}
=
\frac{C_h^{(t)}}{M}
\overline{\Delta W}_h^{(t)},
\]
whereas raFLoRA gives
\[
\Delta W_{g,h,\mathrm{ra}}^{(t+1)}
=
\overline{\Delta W}_h^{(t)}.
\]
Thus, raFLoRA removes the factor $C_h^{(t)}/M$ responsible for one-round contributor dilution. Rather than amplifying higher-rank updates, it normalizes each partition over its effective contributors.

The same effect appears in the stable-routing model. Fix a direction $i\in\mathcal{I}_h$. Since client ranks take values only from $\{r_1,\ldots,r_H\}$, a client supports direction $i$ if and only if it supports partition $h$, i.e.,
\[
c_{k,i}=\mathbf{1}\{r_k\ge i\}
       =\mathbf{1}\{r_k\ge r_h\},
\qquad i\in\mathcal{I}_h.
\]
Hence,
\[
C_h^{(t)}
=
\sum_{k\in\mathcal{M}_t} c_{k,i},
\qquad i\in\mathcal{I}_h.
\]

Conventional aggregation uses
\[
a_{i,\mathrm{base}}^{(t)}
=
\frac{1}{M}
\sum_{k\in\mathcal{M}_t}
c_{k,i}\beta_{k,i}^{(t)},
\]
with
\[
\mathbb{E}[a_{i,\mathrm{base}}^{(t)}]
=
p_i\mu_i^{(t)}.
\]
For $C_h^{(t)}>0$, raFLoRA instead uses
\[
a_{i,\mathrm{ra}}^{(t)}
=
\frac{1}{C_h^{(t)}}
\sum_{k\in\mathcal{M}_t}
c_{k,i}\beta_{k,i}^{(t)},
\]
so that
\[
\mathbb{E}
\left[
a_{i,\mathrm{ra}}^{(t)}
\mid C_h^{(t)}>0
\right]
=
\mu_i^{(t)}.
\]
Hence, conditional on receiving an update, the coverage-dependent attenuation $p_i\mu_i^{(t)}$ is removed.

For illustration, suppose all eligible clients have the same directional gain $\beta_i^{(t)}$. Conventional aggregation gives
\[
q_{i,\mathrm{base}}^{(t)}
=
(\beta_i^{(t)})^2
\left[(1-\tau)p_i^2+\tau p_i\right],
\]
whereas, conditional on $C_h^{(t)}>0$, raFLoRA gives
\[
q_{i,\mathrm{ra}}^{(t)}
=
(\beta_i^{(t)})^2.
\]
Thus, the systematic coverage penalty that creates the retention gap in Theorem~\ref{thm:cumulative_concentration} is removed in nonempty rounds.

When $C_h^{(t)}=0$, raFLoRA retains the previous global partition instead of replacing it with zero. Under the same simplified model, the corresponding multiplier is therefore one. Thus, for equal directional gains, the unconditional multiplier becomes
\[
q_{i,\mathrm{ra}}^{(t)}
=
\Pr(C_h^{(t)}>0)(\beta_i^{(t)})^2
+
\Pr(C_h^{(t)}=0).
\]
The fallback introduces no new information, but prevents an empty-contributor round from discarding previously retained higher-rank information.

Overall, raFLoRA removes the systematic contributor-dependent attenuation identified by our analysis and preserves existing information in empty partitions. Other effects, including SVD reordering and spectral mixing, may still affect the multi-round spectrum.
\section{Additional Experiments}
\label{appendix:extended-exp}
To further examine the impact of rank heterogeneity and LoRA module insertion, we conduct extended experiments that vary both rank configurations and LoRA insertions.

\begin{figure}[ht]
  \centering
  \begin{subfigure}{0.49\linewidth}
    \centering
    \includegraphics[width=\columnwidth]{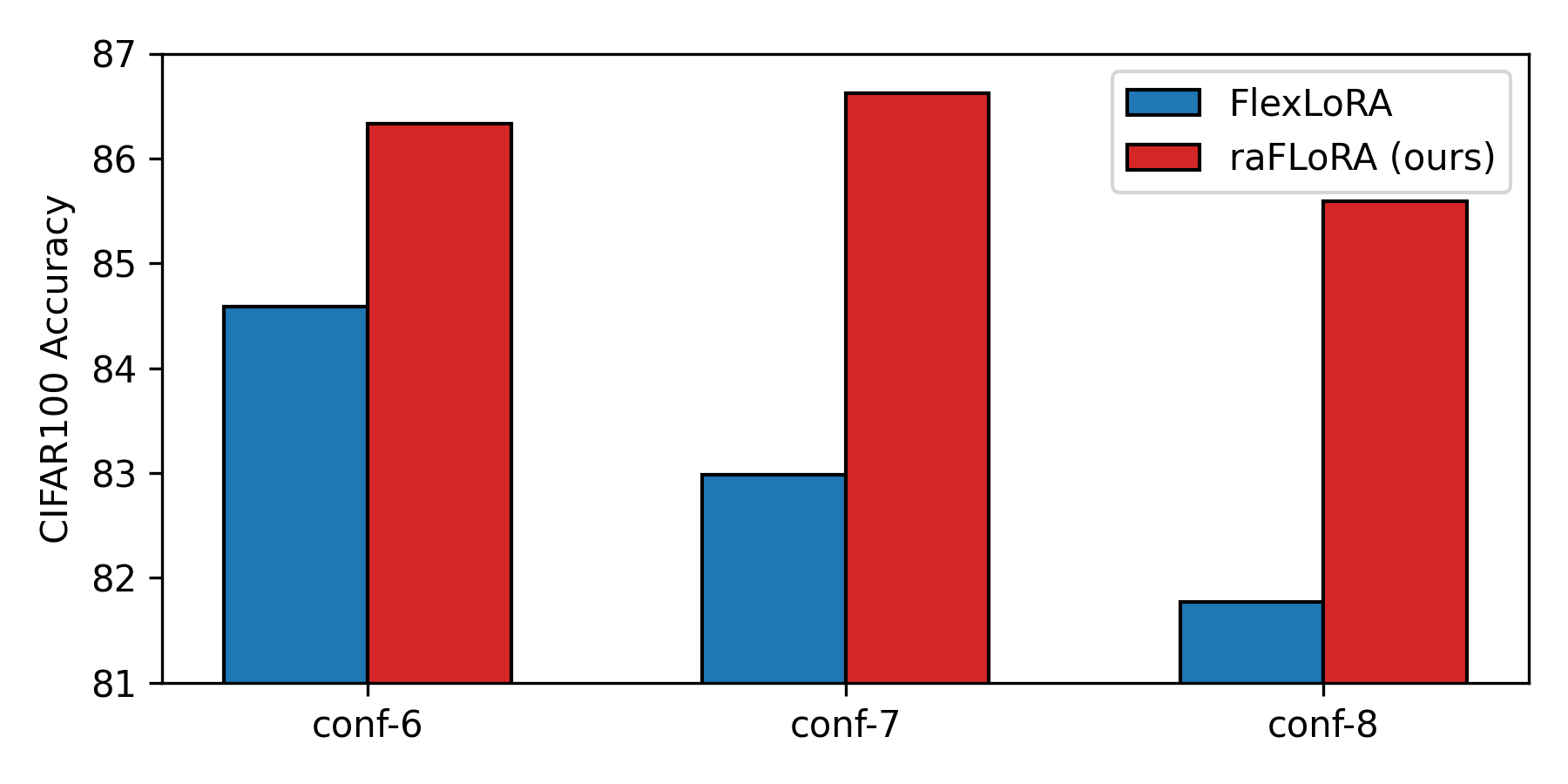}
    \caption{Additional rank configurations}
    \label{fig:extended-ranks}
  \end{subfigure}
  \begin{subfigure}{0.49\linewidth}
    \centering
    \includegraphics[width=\columnwidth]{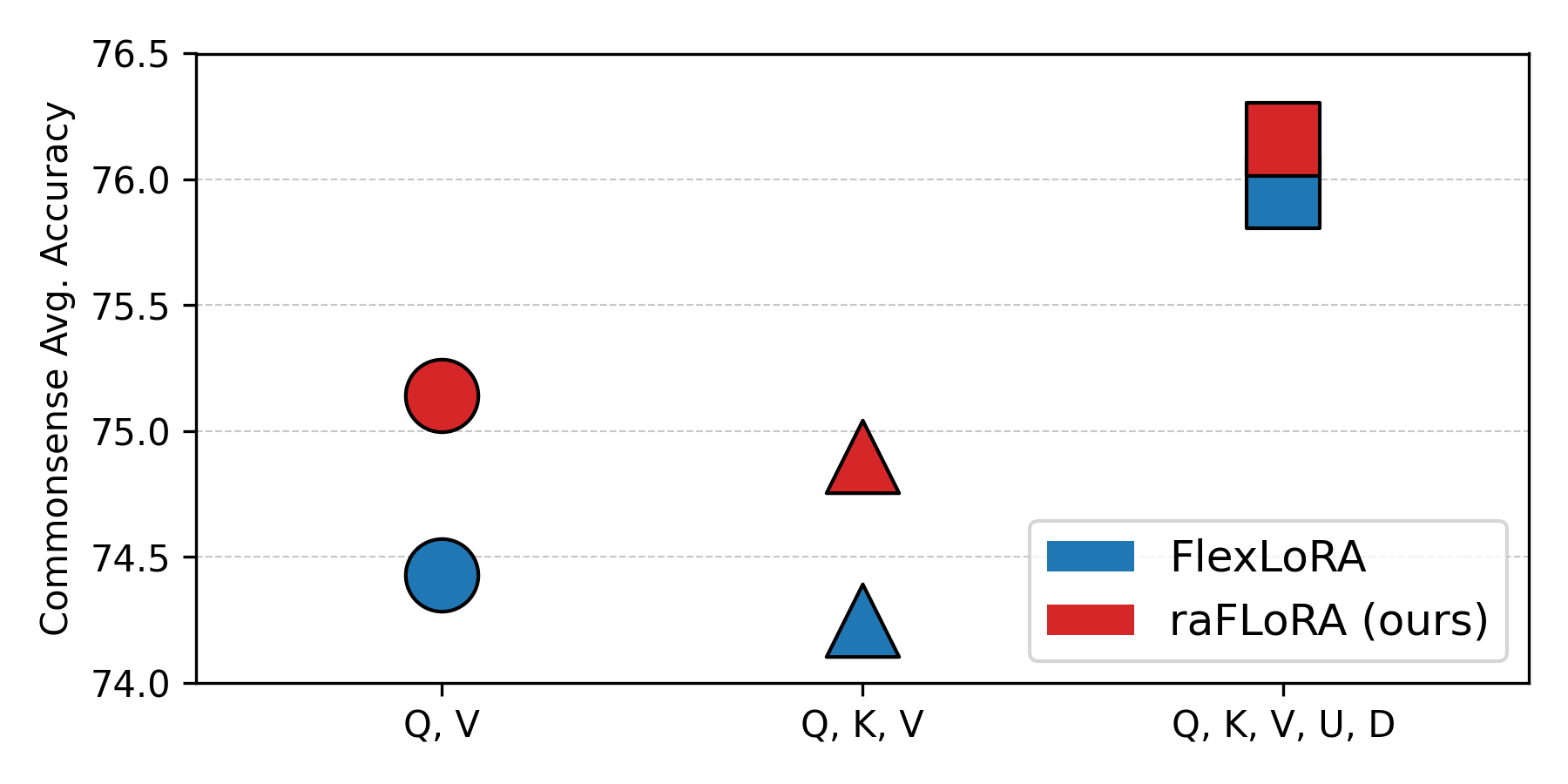}
    \caption{Different LoRA insertion modules}
    \label{fig:extended-modules}
  \end{subfigure}
  \caption{Extended experiments comparing raFLoRA and FlexLoRA under varying rank configurations and LoRA module insertion settings.}
  \label{fig:extended-exp}
\end{figure}

\paragraph{Effect of additional rank configurations.}
We evaluate a broader range of rank configurations, including conf-6 $\{8,12,16,20,24\}$, conf-7 $\{4,8,16,32,64\}$, and conf-8 $\{1,4,16,64,256\}$, which exhibit progressively larger rank gaps.
As shown in Figure~\ref{fig:extended-ranks}, raFLoRA consistently achieves larger accuracy improvements over FlexLoRA as rank heterogeneity increases, with gains of up to 4\%.
This trend reflects the strong dependence of FlexLoRA on the minimum client rank, which increasingly constrains the effective rank of the global update under larger rank gaps.
In contrast, raFLoRA mitigates rank collapse through rank-partitioned aggregation, enabling higher-rank clients to contribute more effectively and thereby better exploiting larger rank configurations.

\paragraph{Effect of different LoRA module insertion settings.}
We further evaluate the robustness of raFLoRA under various LoRA module insertion settings using LLaMA-3.2-3B, including $\{Q,V\}$, $\{Q,K,V\}$, and $\{Q,K,V,U,D\}$, where $Q$, $K$, and $V$ denote the query, key, and value projections in attention layers, and $U$ and $D$ correspond to the MLP up- and down-projection layers, respectively.
These configurations progressively increase the number of adapted modules.
As shown in Figure~\ref{fig:extended-modules}, raFLoRA consistently achieves higher average accuracy across all commonsense reasoning benchmarks under these settings.
These results indicate that raFLoRA scales well with the number of inserted LoRA modules and remains robust across diverse module configurations, highlighting its suitability for flexible and scalable adaptation of large models.


\section{Training Dynamics of Accuracy and Loss}
\label{appendix:training-dynamics}

We track the global evaluation accuracy and global training loss during fine-tuning.
As illustrated in Figure~\ref{fig:acc-loss}, raFLoRA consistently achieves higher global accuracy and lower global loss across training rounds on both ViT-base and LLaMA-3.2-3B.
Notably, FLoRA re-initializes its LoRA parameters at each round of local training.
Although global updates are merged into the base model, the low-rank adaptation does not persist across rounds, causing the optimization trajectory in the low-rank subspace to be repeatedly reset.
This may slow convergence and degrade performance on more complex tasks.

\begin{figure}[ht]
  \centering
  \begin{subfigure}{0.49\linewidth}
    \centering
    \includegraphics[width=\linewidth]{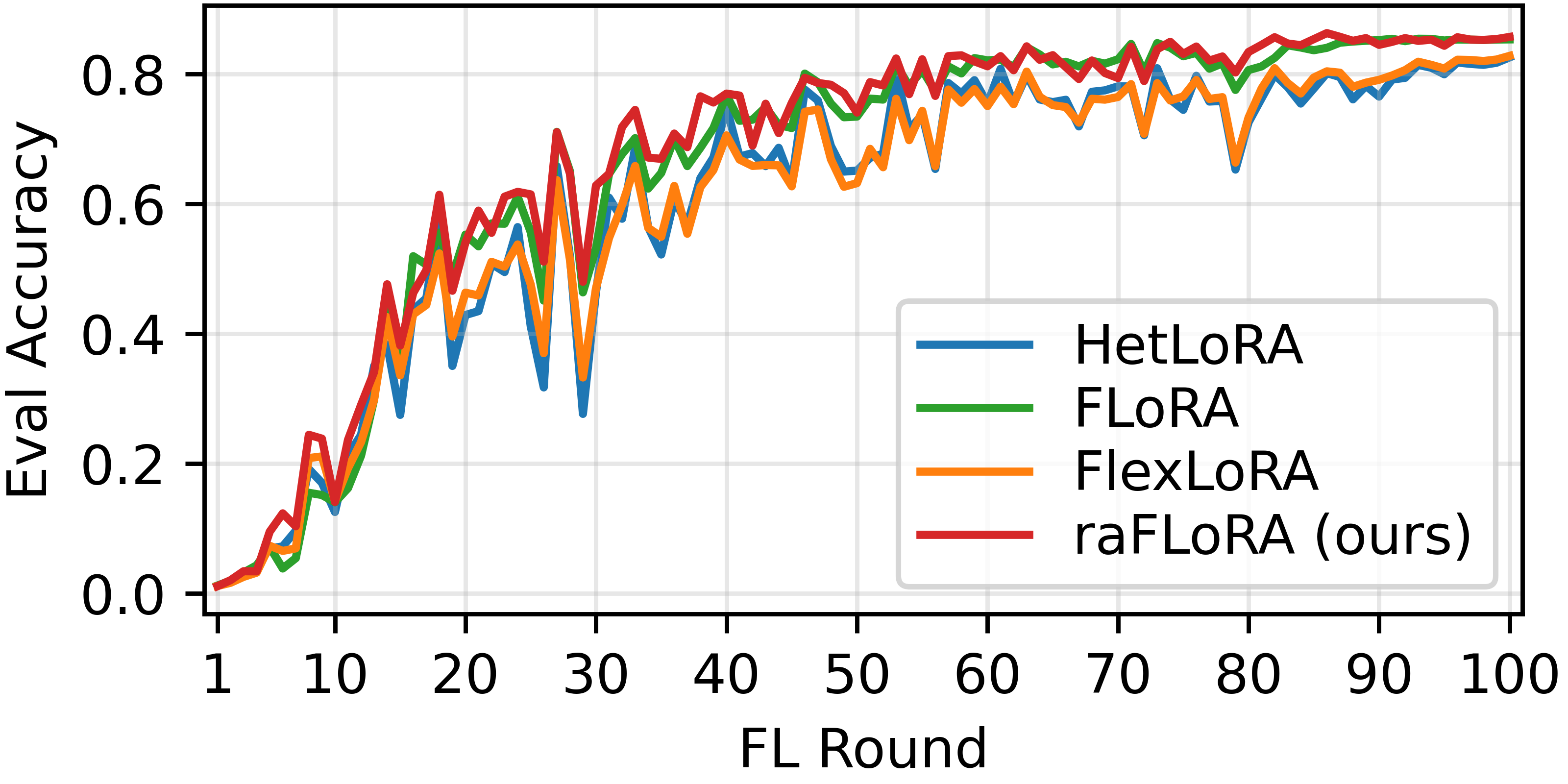}
    \caption{Accuracy (ViT-base)}
    \label{fig:vit-acc}
  \end{subfigure}
  \hfill
  \begin{subfigure}{0.49\linewidth}
    \centering
    \includegraphics[width=\linewidth]{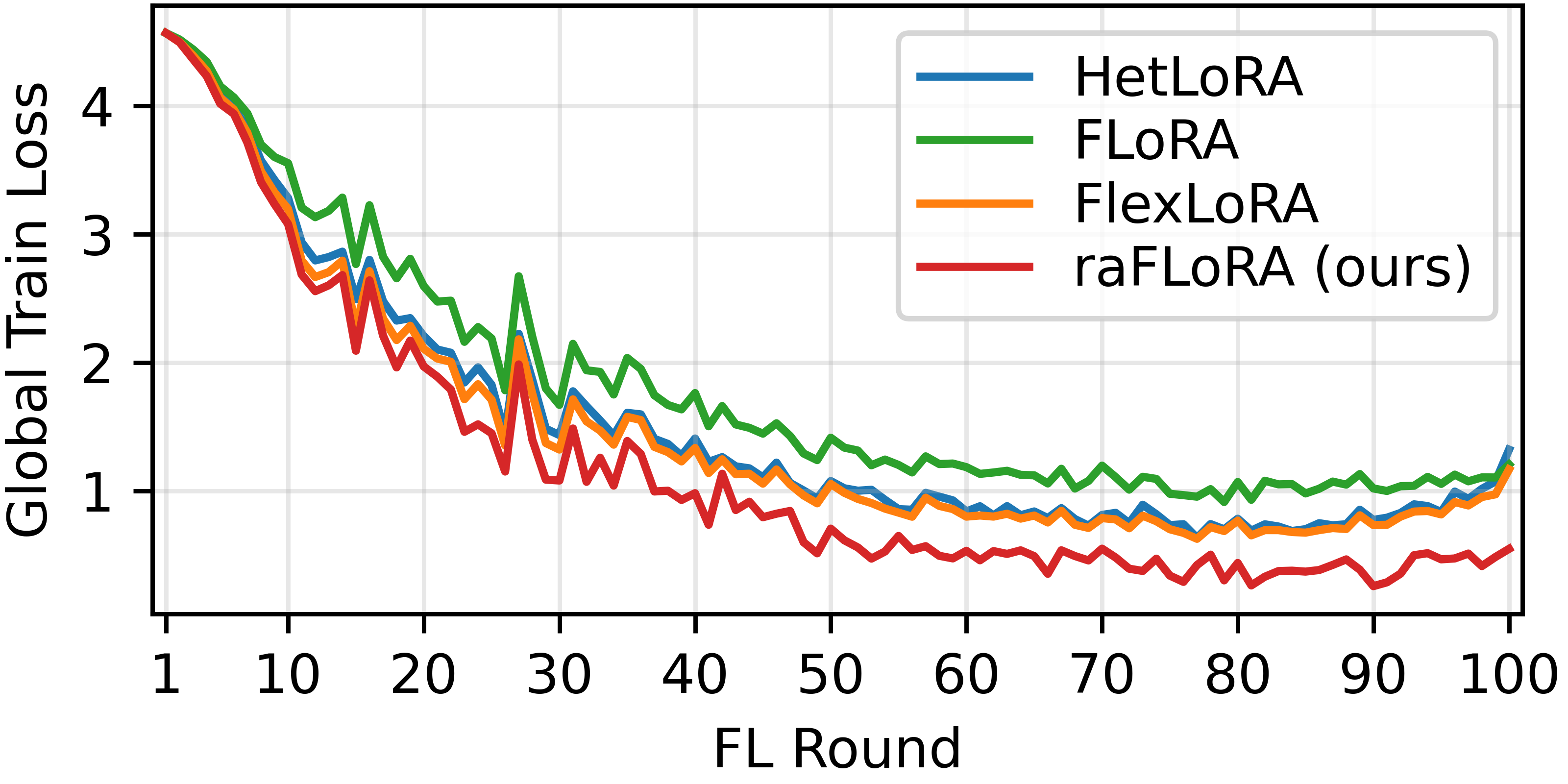}
    \caption{Loss (ViT-base)}
    \label{fig:vit-loss}
  \end{subfigure}

  \vspace{0.02in}

  \begin{subfigure}{0.49\linewidth}
    \centering
    \includegraphics[width=\linewidth]{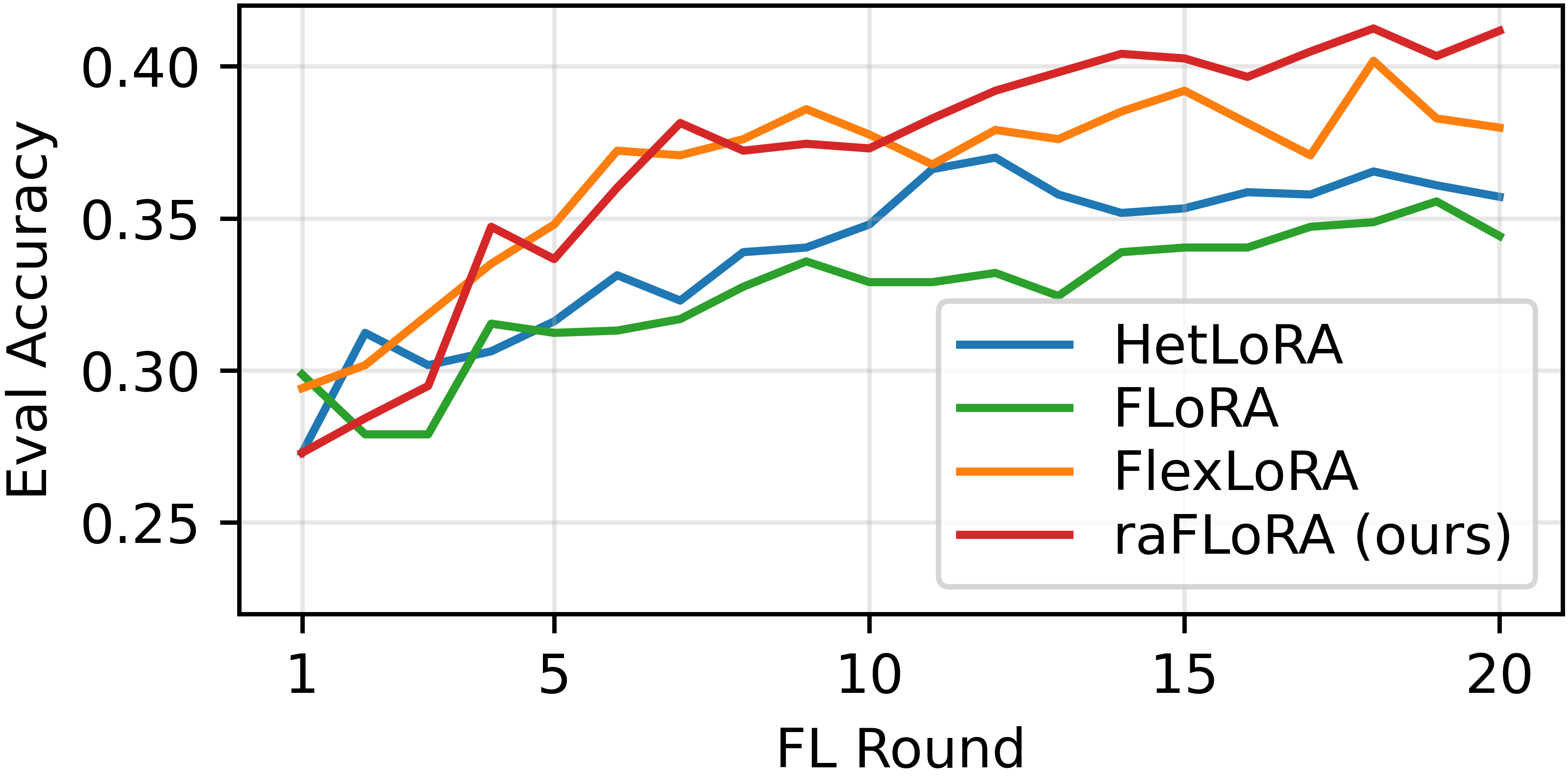}
    \caption{Accuracy (LLaMA-3.2-3B)}
    \label{fig:llama-acc}
  \end{subfigure}
  \hfill
  \begin{subfigure}{0.49\linewidth}
    \centering
    \includegraphics[width=\linewidth]{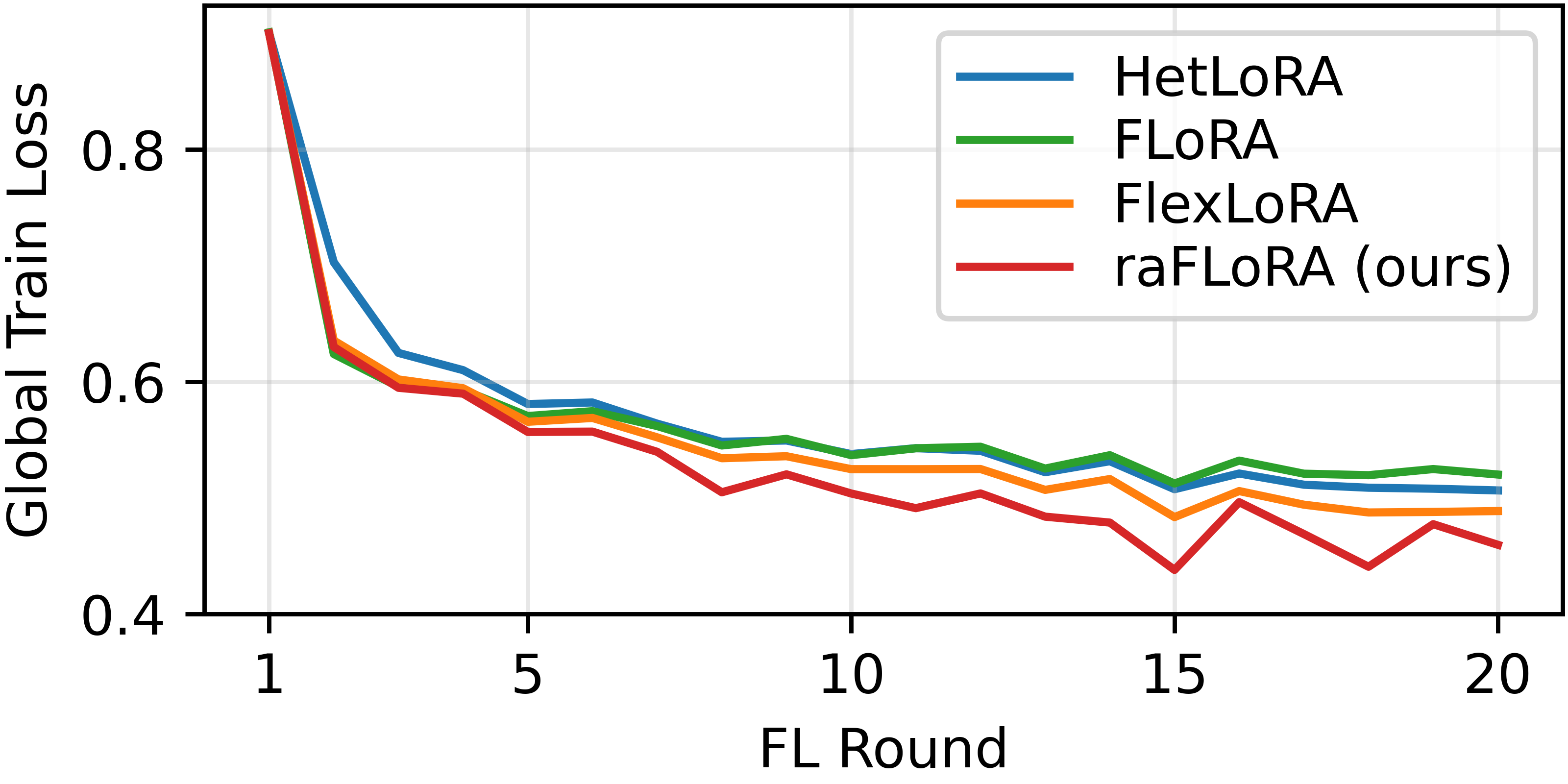}
    \caption{Loss (LLaMA-3.2-3B)}
    \label{fig:llama-loss}
  \end{subfigure}

  \caption{Global evaluation accuracy and global training loss over communication rounds. Results on CIFAR100 with ViT-base are shown in (a) and (b), while results on GSM8K with LLaMA-3.2-3B are shown in (c) and (d).}
  \label{fig:acc-loss}
\end{figure}



\section{Hyperparameter Settings for Main Experiments}
\label{appendix:hyperparameters}

\subsection{Image Classification for Vision}
\label{appendix:vision}
\paragraph{Datasets}
For image classification, we conduct experiments on CIFAR100~\cite{Krizhevsky2009cifar100}, which contains 60K color images of size $32 \times 32$ from 100 classes. Each class includes 500 training samples and 100 test samples, resulting in 50K training images and 10K test images in total.

\begin{table}[ht]
\caption{Hyperparameter settings for image classification experiments on CIFAR100.}
\label{tab:hyper-cifar100}
\centering
\small
\begin{tabular}{lc}
\toprule
\textbf{Hyperparameters} & \textbf{Values} \\
\midrule
Number of Clients                       & 100 \\
Number of Rounds                        & 100 \\
Client Participation Ratio per Round    & 10\% \\
Data Partitioning                       & c20($\alpha=1$) \\
Local Training Epoch                    & 1 \\
Batch Size                              & 32 \\
Optimizer                               & AdamW \\
Learning Rate                           & 5e-4 \\
Learning Rate Scheduler                 & Linear decay per round \\
Inserted Modules of LoRA                & All linear layers \\
LoRA Rank Configurations                & $\{8,16,32,48,64\}$ \\
Rank Probability Distributions          & $\{0.2,0.2,0.2,0.2,0.2\}$ \\
\bottomrule
\end{tabular}
\end{table}

\paragraph{Hyperparameters}
For image classification, we follow standard federated learning configurations.
Unless otherwise specified, all experiments adopt identical optimization and communication settings, with the complete hyperparameter configuration summarized in Table~\ref{tab:hyper-cifar100}.
Given its relatively large label space, CIFAR100 supports flexible non-IID data partitioning, and we therefore fix its data partition to the pathological non-IID setting c20($\alpha=1$) in the main experiments.

For the experiments in Section~\ref{sec:intro}, Figures~\ref{fig:energy-flexlora} and~\ref{fig:energy-raflora} use the same configuration as Table~\ref{tab:hyper-cifar100}. Figure~\ref{fig:acc-niid} changes only the data partitioning scheme to the regular non-IID setting, with the LoRA rank selected from $\{1,16,32,48,64\}$. Figure~\ref{fig:acc-config} changes only the shared rank $r_1$, considering three rank sets: $\{1,16,32,48,64\}$ for $r_1=1$, $\{4,16,32,48,64\}$ for $r_1=4$, and $\{8,16,32,48,64\}$ for $r_1=8$.

\subsection{Text Classification for Language}
\label{appendix:language}
\paragraph{Datasets}
For text classification, we evaluate our method on the 20 Newsgroups~\cite{lang199520ng} dataset, a topic classification benchmark consisting of newsgroup posts spanning 20 distinct categories. 
The dataset contains approximately 11.3K training samples and 7.5K test samples.

\begin{table}[ht]
\caption{Hyperparameter settings for text classification experiments on 20NG.}
\label{tab:hyper-tc}
\centering
\small
\begin{tabular}{lc}
\toprule
\textbf{Hyperparameters} & \textbf{Values} \\
\midrule
Number of Clients                       & 100 \\
Number of Rounds                        & 100 \\
Client Participation Ratio per Round    & 10\% \\
Data Partitioning                       & c5($\alpha=1$) \\
Local Training Epoch                    & 1 \\
Batch Size                              & 32 \\
Optimizer                               & AdamW \\
Learning Rate                           & 5e-4 \\
Learning Rate Scheduler                 & Linear decay per round \\
Inserted Modules of LoRA                & All linear layers \\
LoRA Rank Configurations                & $\{8,16,32,48,64\}$ \\
Rank Probability Distributions          & $\{0.2,0.2,0.2,0.2,0.2\}$ \\
\bottomrule
\end{tabular}
\end{table}

\paragraph{Hyperparameters}
For text classification, we similarly follow standard federated learning configurations.
Unless otherwise specified, all experiments adopt identical optimization and communication settings, with the complete hyperparameter configuration summarized in Table~\ref{tab:hyper-tc}.
Owing to its larger number of classes, the 20 Newsgroups dataset allows flexible non-IID partitioning, and we fix its data partition to the pathological non-IID setting c5($\alpha=1$) in the main experiments.

\subsection{Mathematical Reasoning}
\label{appendix:math}
\paragraph{Datasets}
For mathematical reasoning, we use the GSM8K~\cite{cobbe2021gsm8k} dataset, which contains approximately 8.5K high-quality, linguistically diverse grade-school math word problems, including 7.5K training samples and 1.3K test samples. The task focuses on question answering for basic mathematical problems that require multi-step reasoning, typically involving 2 to 8 solution steps. The problems require no concepts beyond early algebra and are primarily solved through sequences of elementary arithmetic operations (e.g., \(+\), \( - \), \( \times \), \( \div \)). Each solution is provided in natural language rather than symbolic expressions, making the dataset particularly suitable for evaluating step-by-step reasoning capabilities of language models.

\begin{table}[ht]
\caption{Hyperparameter settings for mathematical reasoning experiments on GSM8K.}
\label{tab:hyper-math}
\centering
\small
\begin{tabular}{lc}
\toprule
\textbf{Hyperparameters} & \textbf{Values} \\
\midrule
Number of Clients                       & 100 \\
Number of Rounds                        & 20 \\
Client Participation Ratio per Round    & 10\% \\
Data Partitioning                       & iid \\
Local Training Epoch                    & 1 \\
Batch Size                              & 4 \\
Optimizer                               & AdamW \\
Learning Rate                           & 5e-4 \\
Learning Rate Scheduler                 & Linear decay per round \\
Inserted Modules of LoRA                & Query and Value \\
LoRA Rank Configurations                & $\{8,16,32,48,64\}$ \\
Rank Probability Distributions          & $\{0.2,0.2,0.2,0.2,0.2\}$ \\
\bottomrule
\end{tabular}
\end{table}

\paragraph{Hyperparameters}
Following prior work on federated low-rank adaptation of large-scale models, such as Fed-SB~\cite{singhal2025fedsb} and FedEx-LoRA~\cite{singhal2025fedexlora}, we adopt similar training protocols with necessary modifications to accommodate heterogeneous LoRA rank settings. The detailed hyperparameter configurations used in our main experiments are summarized in Table~\ref{tab:hyper-math}. For the GSM8K dataset, we evenly partition the training data across all clients.

\subsection{Commonsense Reasoning}
\label{appendix:common}
\paragraph{Datasets}
For commonsense reasoning, we use the Commonsense15K~\cite{hu2023commonsense} benchmark, which comprises eight sub-tasks: BoolQ~\citep{clark2019boolq}, PIQA~\citep{bisk2020piqa}, SIQA~\citep{sap2019socialiqa}, HellaSwag~\citep{zellers2019hellaswag}, Winogrande~\citep{sakaguchi2021winogrande}, ARC-Easy and ARC-Challenge~\citep{clark2018think}, and OpenBookQA~\citep{mihaylov2018can}.  
Models are fine-tuned on Commonsense15K and evaluated separately on the test sets of each of the eight tasks.
Since these sub-tasks involve different discrete answer formats, such as True/False, Answer, Solution, Option, and Ending, we use the answer-format identifiers as categorical labels to construct non-IID client partitions.

\begin{table}[ht]
\caption{Hyperparameter settings for commonsense reasoning experiments on Commonsense15K.}
\label{tab:hyper-common}
\centering
\small
\begin{tabular}{lc}
\toprule
\textbf{Hyperparameters} & \textbf{Values} \\
\midrule
Number of Clients                       & 100 \\
Number of Rounds                        & 20 \\
Client Participation Ratio per Round    & 10\% \\
Data Partitioning                       & $\alpha=0.5$ \\
Local Training Epoch                    & 1 \\
Batch Size                              & 16 \\
Optimizer                               & AdamW \\
Learning Rate                           & 5e-4 \\
Learning Rate Scheduler                 & Linear decay per round \\
Inserted Modules of LoRA                & Query and Value \\
LoRA Rank Configurations                & $\{8,16,32,48,64\}$ \\
Rank Probability Distributions          & $\{0.2,0.2,0.2,0.2,0.2\}$ \\
\bottomrule
\end{tabular}
\end{table}

\paragraph{Hyperparameters}
Building on established protocols for federated low-rank adaptation of large-scale models~\cite{singhal2025fedexlora,singhal2025fedsb}, we adopt a consistent training setup with minimal adjustments to support heterogeneous LoRA rank settings.
The complete set of hyperparameters is reported in Table~\ref{tab:hyper-common}.
For Commonsense15K, we apply Dirichlet partitioning over the answer-format labels and fix the partition to the $\alpha=0.5$ non-IID setting in the main experiments.
This yields a controlled and reproducible label-skew non-IID split, where different clients are biased toward different answer-format categories and thus receive different mixtures of the underlying sub-tasks.

\section{Detailed Results on Commonsense Reasoning}
\label{appendix:commonsense-results}
For commonsense reasoning, we use Commonsense15K~\cite{hu2023commonsense}, a benchmark including eight datasets: BoolQ~\citep{clark2019boolq}, PIQA~\citep{bisk2020piqa}, SIQA~\citep{sap2019socialiqa}, HellaSwag~\citep{zellers2019hellaswag}, Winogrande~\citep{sakaguchi2021winogrande}, ARC-Easy and ARC-Challenge~\citep{clark2018think}, and OpenBookQA~\citep{mihaylov2018can}.
We fine-tune the global model on Commonsense15K in a federated setting using LLaMA-3.2-3B~\cite{meta2024llama3.2} and LLaMA-3.1-8B~\cite{grattafiori2024llama3}, and evaluate the fine-tuned global model on the test sets of all eight sub-tasks.

As shown in Tables~\ref{tab:main-common-3B} and~\ref{tab:main-common-8B}, raFLoRA achieves the highest average accuracy across the eight commonsense reasoning tasks.
The Avg. column is computed by first averaging the results over the eight sub-tasks for each random seed, and then reporting the mean and standard deviation across three random seeds.
This metric reflects the overall capability of the federated fine-tuned global model, as it summarizes generalization across diverse commonsense reasoning tasks rather than performance on a single sub-task. raFLoRA consistently outperforms HetLoRA and FLoRA on both model scales. Compared with the strong baseline FlexLoRA, raFLoRA achieves higher average accuracy while maintaining on-par or superior performance on all sub-tasks except OBQA.
These results demonstrate the benefit of rank-partitioned aggregation over existing heterogeneous-rank FedLoRA methods.

\begin{table}[t]
\caption{Performance comparison on commonsense reasoning tasks using LLaMA-3.2-3B. }
\label{tab:main-common-3B}
\centering
\small
\renewcommand{\arraystretch}{1.0}
\resizebox{\textwidth}{!}{%
\begin{tabular}{lccccccccc}
\toprule
\textbf{Methods} & \textbf{BoolQ} & \textbf{PIQA} & \textbf{SIQA} & \textbf{HS} & \textbf{WG} & \textbf{ARC-e} & \textbf{ARC-c} & \textbf{OBQA} & \colorbox{lightorange}{\textbf{Avg.}} \\
\midrule
HetLoRA         & 63.78$_{\pm0.23}$ & 79.60$_{\pm1.04}$ & 67.59$_{\pm0.58}$ & 77.54$_{\pm0.71}$ & 60.09$_{\pm1.03}$ & 84.92$_{\pm0.24}$ & 70.36$_{\pm0.47}$ & 71.20$_{\pm1.39}$ & \colorbox{lightorange}{71.89$_{\pm0.41}$} \\
FLoRA           & 62.94$_{\pm0.56}$ & 79.61$_{\pm0.68}$ & 66.79$_{\pm0.58}$ & 74.85$_{\pm1.16}$ & 58.17$_{\pm0.21}$ & 84.33$_{\pm0.07}$ & 70.39$_{\pm0.44}$ & 71.07$_{\pm0.42}$ & \colorbox{lightorange}{71.02$_{\pm0.18}$} \\
FlexLoRA        & 62.44$_{\pm2.82}$ & 81.11$_{\pm0.49}$ & 70.13$_{\pm0.50}$ & 81.76$_{\pm0.84}$ & 64.30$_{\pm0.56}$ & 87.00$_{\pm0.34}$ & 72.55$_{\pm0.73}$ & 75.33$_{\pm1.10}$  & \colorbox{lightorange}{74.33$_{\pm0.50}$} \\
raFLoRA (ours)  & 64.98$_{\pm1.94}$ & 81.10$_{\pm0.35}$ & 70.79$_{\pm1.47}$ & 82.58$_{\pm1.16}$ & 65.40$_{\pm0.81}$ & 86.95$_{\pm0.59}$ & 72.27$_{\pm0.31}$ & 74.80$_{\pm1.25}$  & \colorbox{lightorange}{\textbf{74.86}$_{\pm0.39}$} \\
\bottomrule
\end{tabular}
}
\end{table}

\begin{table}[t]
\caption{Performance comparison on commonsense reasoning tasks using LLaMA-3.1-8B. }
\label{tab:main-common-8B}
\centering
\small
\renewcommand{\arraystretch}{1.0}
\resizebox{\textwidth}{!}{%
\begin{tabular}{lccccccccc}
\toprule
\textbf{Methods} & \textbf{BoolQ} & \textbf{PIQA} & \textbf{SIQA} & \textbf{HS} & \textbf{WG} & \textbf{ARC-e} & \textbf{ARC-c} & \textbf{OBQA} & \colorbox{lightorange}{\textbf{Avg.}} \\
\midrule
HetLoRA         & 70.61$_{\pm0.17}$ & 84.80$_{\pm0.93}$ & 72.84$_{\pm1.54}$ & 86.58$_{\pm1.26}$ & 69.56$_{\pm1.60}$ & 91.63$_{\pm0.51}$ & 78.78$_{\pm0.44}$ & 80.80$_{\pm1.73}$ & \colorbox{lightorange}{79.45$_{\pm1.00}$} \\
FLoRA           & 69.67$_{\pm1.78}$ & 85.29$_{\pm0.33}$ & 72.96$_{\pm0.31}$ & 86.36$_{\pm0.57}$ & 66.72$_{\pm0.77}$ & 91.65$_{\pm0.16}$ & 79.15$_{\pm0.30}$ & 80.27$_{\pm0.90}$ & \colorbox{lightorange}{79.01$_{\pm0.38}$} \\
FlexLoRA        & 70.76$_{\pm0.89}$ & 85.76$_{\pm1.28}$ & 73.70$_{\pm1.47}$ & 88.99$_{\pm0.57}$ & 73.56$_{\pm0.55}$ & 92.06$_{\pm0.50}$ & 79.49$_{\pm0.70}$ & 82.73$_{\pm1.53}$ & \colorbox{lightorange}{80.88$_{\pm0.67}$} \\
raFLoRA (ours)  & 70.88$_{\pm1.15}$ & 85.38$_{\pm1.39}$ & 74.65$_{\pm0.73}$ & 89.06$_{\pm0.31}$ & 74.74$_{\pm0.48}$ & 91.96$_{\pm0.34}$ & 80.66$_{\pm0.13}$ & 81.87$_{\pm1.10}$ & \colorbox{lightorange}{\textbf{81.15}$_{\pm0.08}$} \\
\bottomrule
\end{tabular}
}
\end{table}

\section{Generality of the Rank-Wise Aggregation Mismatch}
\label{appendix:generality}

Our theoretical analysis targets heterogeneous FedLoRA methods that reconstruct client updates, aggregate them using rank-agnostic weights, and subsequently apply SVD-based heterogeneous rank allocation. FlexLoRA serves as the representative method in the main experiments. We next examine whether the identified rank-wise aggregation mismatch is specific to FlexLoRA.

\subsection{Another SVD-Based Heterogeneous FedLoRA Method}

We further examine whether rank collapse also arises in ECLoRA, which is another heterogeneous FedLoRA method with SVD-based rank allocation. During reproduction, we identified an implementation discrepancy in the released ECLoRA code, where RSVD appears to be applied to the uncompensated update rather than the compensated update specified in the paper. We corrected this issue and report all ECLoRA results below using the corrected implementation.

\begin{table}[h]
\centering
\caption{Energy distribution and accuracy after 100 FL rounds for ECLoRA with and without rank-partitioned aggregation on CIFAR100 using ViT-base.}
\label{tab:eclora_energy_cifar}
\small
\resizebox{\linewidth}{!}{
\begin{tabular}{lccccc|c}
\toprule
Method & Rank 1-8 (\%) & Rank 9-16 (\%) & Rank 17-32 (\%) & Rank 33-48 (\%) & Rank 49-64 (\%) & Accuracy (\%) \\
\midrule
ECLoRA & 99.9986 & 0.0012 & 0.0001 & 0.0000 & 0.0000 & 84.07 \\
ECLoRA + raFLoRA & 38.2581 & 22.6282 & 26.9840 & 12.1232 & 0.0165 & \textbf{86.42} \\
\bottomrule
\end{tabular}
}
\end{table}

As shown in Table~\ref{tab:eclora_energy_cifar}, ECLoRA also exhibits energy concentration, with almost all singular-value energy remaining within ranks 1-8. Applying rank-partitioned weighting broadens the energy distribution across higher-rank components and improves accuracy from 84.07\% to 86.42\%.

These results show that the observed rank collapse is not specific to FlexLoRA and can also arise in other heterogeneous FedLoRA methods with SVD-based rank allocation. They further suggest that rank-partitioned aggregation can benefit such methods.

\subsection{Extension Beyond the Analyzed SVD-Based Framework}

We further examine whether accounting for rank-dependent client support can benefit heterogeneous FedLoRA methods beyond the specific SVD-based framework considered in our theory. We apply the same rank-partitioned weighting principle to HetLoRA and FLoRA under the same experimental setting.

\begin{table}[t]
\centering
\caption{Applying rank-partitioned weighting to heterogeneous FedLoRA methods outside the SVD-based framework.}
\label{tab:general_weighting}
\small
\begin{tabular}{lc}
\toprule
Method & Accuracy (\%) \\
\midrule
HetLoRA & 84.30 \\
HetLoRA + raFLoRA & \textbf{84.72} \\
FLoRA & 86.92 \\
FLoRA + raFLoRA & \textbf{87.26} \\
\bottomrule
\end{tabular}
\end{table}

As shown in Table~\ref{tab:general_weighting}, rank-partitioned weighting improves HetLoRA from 84.30\% to 84.72\% and FLoRA from 86.92\% to 87.26\%. Although these methods fall outside the scope of Theorem~\ref{thm:cumulative_concentration}, the consistent improvements suggest that explicitly accounting for rank-dependent client support can also benefit heterogeneous FedLoRA methods beyond the SVD-based setting analyzed in our theory. 

Overall, these results indicate that the aggregation principle identified by our analysis is not specific to FlexLoRA and may generalize to a broader range of heterogeneous FedLoRA methods.

\section{Post-SVD Coverage Correction}
\label{appendix:post_svd_scaling}

A natural alternative to rank-partitioned aggregation is to compensate for rank-dependent attenuation after server-side SVD. Since higher-rank directions have lower coverage $p_i$, one may directly rescale the $i$-th singular component by $1/p_i$. Specifically, given the aggregated update
\[
\Delta W_g
=
U\Sigma V^\top,
\qquad
\Sigma=\operatorname{diag}(\sigma_1,\ldots,\sigma_{r_{\max}}),
\]
the post-SVD correction replaces each singular value by
\[
\widehat{\sigma}_i
=
\frac{\sigma_i}{p_i}.
\]

Although this correction appears consistent with the coverage-dependent attenuation identified in our analysis, it differs fundamentally from raFLoRA. The post-SVD component at position $i$ does not generally correspond to the pre-SVD rank partition with coverage $p_i$. Dense aggregation can mix client updates, while SVD can rotate and reorder the resulting spectral directions. Therefore, multiplying the $i$-th post-SVD component by $1/p_i$ corrects an SVD position rather than the original rank partition supported by a particular set of clients. In contrast, raFLoRA performs contributor-aware normalization before SVD, when the correspondence between each rank partition and its effective contributors is still explicit.

We empirically evaluate this alternative correction on CIFAR100 under the same setting used in our main analysis. Table~\ref{tab:post_svd_scaling} reports the final accuracy, higher-rank energy ratio, and relative Frobenius aggregation gap.

\begin{table}[h]
\centering
\caption{Comparison between conventional aggregation and post-SVD
$1/p_i$ scaling.}
\label{tab:post_svd_scaling}
\small
\begin{tabular}{lccc}
\toprule
Method
& Accuracy (\%)
& Higher-rank energy (\%)
& Frobenius gap (\%) \\
\midrule
FlexLoRA
& 84.21\%
& $\sim$0.001\%
& 0.00014\% \\
Post-SVD $1/p_i$ scaling
& 86.73\%
& 10.4\%
& 8.99\% \\
\bottomrule
\end{tabular}
\end{table}

Post-SVD scaling substantially increases the higher-rank energy ratio from approximately $0.001\%$ to $10.4\%$ and improves accuracy from 84.21\% to 86.73\%. This result provides additional evidence that restoring suppressed higher-rank capacity can translate into meaningful performance gains. However, the correction also increases the relative Frobenius aggregation gap from 0.00014\% to 8.99\%, indicating that rescaling the post-SVD spectrum can distort the aggregated update.

Moreover, its accuracy remains below the 87.45\% achieved by raFLoRA under the same setting. These results suggest that correcting rank-dependent attenuation is more effective when performed at the contributor level before SVD rather than by rescaling spectral positions after SVD. raFLoRA therefore addresses the source of the mismatch directly, while post-SVD scaling provides only an approximate correction after spectral mixing and reordering have already occurred.

\section{Limitations}
\label{sec:limitations}

Although raFLoRA mitigates rank collapse by correcting the rank-wise aggregation mismatch, preserving higher-rank capacity does not imply that every higher-rank component is useful. To examine the functional contribution of different rank ranges, we perform a rank-prefix ablation by retaining only the first $r$ components of the final global update.

\begin{table}[h]
\centering
\caption{Rank-prefix ablation on CIFAR100 using ViT-base. Each row retains only the indicated prefix of the final global LoRA update.}
\label{tab:rank_prefix_ablation}
\small
\begin{tabular}{lcc}
\toprule
Retained ranks & FlexLoRA (\%) & raFLoRA (\%) \\
\midrule
Classifier only, rank 0 & 50.84 & 49.03 \\
Rank 1-8   & 83.82 & 83.07 \\
Rank 1-16  & 83.79 & 86.90 \\
Rank 1-32  & 83.82 & 87.38 \\
Rank 1-48  & 83.83 & 87.25 \\
Rank 1-64  & 83.82 & 87.14 \\
\bottomrule
\end{tabular}
\end{table}

Table~\ref{tab:rank_prefix_ablation} shows that FlexLoRA gains little beyond ranks 1-8, whereas ranks 9-32 substantially improve raFLoRA from 83.07\% to 87.38\%. Extending further to ranks 48 or 64 slightly reduces accuracy, indicating that preserving higher-rank capacity is beneficial overall, although not every additional high-rank component contributes positively.

A remaining limitation is the reliability of sparsely supported partitions. Contributor-aware normalization corrects rank-dependent attenuation but provides limited variance reduction when only a few clients contribute, making single-contributor partitions more sensitive to noisy updates. Future extensions could improve their stability through shrinkage toward the previous global partition or reliability-aware weighting. Thus, raFLoRA addresses rank-dependent attenuation, while variance control for sparsely supported partitions remains a complementary problem.